\documentclass{article}
\usepackage{ijcai24}
\usepackage{natbib}

\usepackage{times}
\usepackage{soul}
\usepackage{url}
\usepackage[hidelinks]{hyperref}
\usepackage[utf8]{inputenc}
\usepackage[small]{caption}
\usepackage{graphicx}
\usepackage{amsmath}
\usepackage{amsthm}
\usepackage{booktabs}
\usepackage{algorithm}
\usepackage{algorithmicx}
\usepackage[switch]{lineno}

\usepackage{amsfonts}
\usepackage{multirow}
\usepackage{multicol}
\usepackage{tabularx}
\usepackage{tabularray}
\usepackage{array}
\usepackage{booktabs}
\usepackage{makecell}

\newtheorem{proposition}{Proposition}

\usepackage{algpseudocode}
\usepackage{circledsteps}
\usepackage{courier}
\usepackage{color,soul}

\usepackage{nameref}
\usepackage{caption}
\usepackage{subcaption}
\usepackage{makecell}
\usepackage{xspace}

\usepackage{subcaption}

\usepackage{tikz}
\usepackage{physics}
\usepackage{cancel}
\usepackage{siunitx}
\usepackage{multirow}
\usepackage{tabularx}

\newcommand{\eg}{\emph{e.g.,}\xspace}
\newcommand{\ie}{\emph{i.e.,}\xspace}

\newcommand{\etr}{\emph{effort-to-recourse}\xspace}
\newcommand{\ttr}{\emph{time-to-recourse}\xspace}
\newcommand{\fairselect}{Circumstance-Normalized Selection\xspace}
\newcommand{\cns}{\emph{CNS}\xspace}
\newcommand{\cda}{\emph{CDA}\xspace}

\title{Fairness in Algorithmic Recourse Through the Lens of Substantive Equality of Opportunity}

\author{
Andrew Bell\footnote{Equal contribution by authors}$^{1}$
\and
Joao Fonseca$^{*2}$\and
Carlo Abrate$^3$\and
Francesco Bonchi$^3$\And
Julia Stoyanovich$^1$\\
\affiliations
$^1$New York University
$^2$NOVA University
$^3$CENTAI
\emails
$^1$\{alb9742, stoyanovich\}@nyu.edu,
$^2$\{jpfonseca\}@novaims.unl.pt,
$^3$\{carlo.abrate, francesco.bonchi\}@centai.eu,
}

\begin{document}

\maketitle

\begin{abstract}
\label{abstract}
Algorithmic recourse---providing recommendations to those affected negatively by the outcome of an algorithmic system on how they can take action and change that outcome---has gained attention as a means of giving persons agency in their interactions with artificial intelligence (AI) systems. Recent work has shown that even if an AI decision-making classifier is ``fair'' (according to some reasonable criteria), recourse \emph{itself} may be unfair due to differences in the initial circumstances of individuals, compounding disparities for marginalized populations and requiring them to exert more effort than others. There is a need to define more methods and metrics for evaluating fairness in recourse that span a range of normative views of the world, and specifically those that take into account time. Time is a critical element in recourse because the longer it takes an individual to act, the more the setting may change due to model or data drift.

This paper seeks to close this research gap by proposing two notions of fairness in recourse that are in normative alignment with substantive equality of opportunity, and that consider time. The first considers the (often repeated) effort individuals exert per successful recourse event, and the second considers time per successful recourse event. Building upon an agent-based framework for simulating recourse, this paper demonstrates how much effort is needed to overcome disparities in initial circumstances. We then proposes an intervention to improve the fairness of recourse by rewarding effort, and compare it to existing strategies.

\end{abstract}

\section{Introduction}
\label{sec:introduction}


In tandem with the rapid proliferation of artificial intelligence (AI) systems in the public and private sectors, there are growing concerns about the significant risks posed by these systems, particularly for marginalized or historically disadvantaged groups~\cite{hu_2020, Sapiezynski2017AcademicPP, obermeyer2019dissecting}. As a result, researchers have given significant attention to assessing and mitigating the unfairness of \emph{predictive classifiers} (\eg~\cite{DBLP:journals/bigdata/Chouldechova17,feldman2015certifying,hardt2016equality,kamiran2009classifying,DBLP:conf/innovations/KleinbergMR17})
and \emph{algorithmic rankers}~\cite{zehlike2022fairness,zehlike2022fairness2}.  

Starting with the seminal work of~\citet{DBLP:conf/innovations/DworkHPRZ12}, it is broadly recognized that notions of algorithmic fairness are not purely technical,  but, rather, that they require thoughtful and explicit normative grounding. When algorithms are used to distribute access to resources or opportunities, then it is appropriate to ground fairness notions in equality of opportunity (EO) doctrines (e.g.,~\cite{DBLP:conf/eaamo/KhanMS22,DBLP:conf/fat/HeidariLGK19,DBLP:conf/innovations/DworkHPRZ12}).


An important idea that has emerged from work in the responsible use of AI is \emph{algorithmic recourse} --- providing recommendations to those receiving the negative outcome on how they can take action and change that outcome. Unfortunately, recent work has revealed that \emph{recourse mechanisms themselves} may be unfair~\cite{gupta2019equalizing,von2022fairness}. In the example below, we illustrate that individuals whose initial circumstances differ due to factors that are outside of their direct control may need to exert a vastly different amount of \emph{effort} to reverse a negative outcome.  This is unfair according to \emph{substantive equality of opportunity}, as we will discuss later in this section.


\paragraph*{Motivating example.}

Imagine Anne and Dawn, two students who repeatedly apply to a gifted program at their school, where their application is evaluated by an algorithm that offers recourse to denied students via recommendations like ``improve your math score by X\%.'' Both are denied the first semester they apply, and both plan to act on the recommendation --- work hard to improve their scores. The two students have different \emph{initial circumstances}, likely due to a socioeconomic difference between them: Ann's score is just below the cut-off for admissions, while Dawn's score is substantially lower.  Anne and Dawn both make \emph{highly significant} improvements to their scores and re-apply.  Ultimately, Anne is accepted into the program on her next attempt, but Dawn is not accepted even after she re-applies three times, substantially improving her score with each attempt.


Naturally, due to limited capacity, there will always be a score cut-off for the program --- but the amount of effort a student exerts may itself be a signal for admission.  


\paragraph*{Research gap.}

One may be tempted to think that unfairness in recourse will not arise if the underlying decision-making process is fair.  Yet, it has been shown that a classifier's fairness does not guarantee that recourse under that classifier will be fair~\cite{gupta2019equalizing, von2022fairness}. This motivates the development of a distinct suite of methods for detecting and mitigating unfairness in recourse that span a range of normative values and contexts. Despite important recent work~\cite{ehyaei2023robustness,gupta2019equalizing, von2022fairness, sharma2019certifai}, questions such as ``What does it mean for recourse to be fair?'' and ``How can we guarantee fair recourse?'' remain partially unanswered. 
Most notably, reasoning about the impact of \emph{time} on fairness in recourse has been overlooked.
Time is intrinsic in this setting: Recourse involves a first unsuccessful attempt at a favorable outcome at time $t$, followed by one or several subsequent attempts at times $t + \delta_1$, $t + \delta_2, ... t + \delta_n$. Continuously changing contexts can weaken the reliability of recourse recommendations over time due to data and model drift~\cite{altmeyer2023endogenous,ferrario2022robustness,DBLP:conf/eaamo/FonsecaBABS23,pawelczyk2022probabilistically}.   

In summary, our goal is to make progress towards filling this gap by \emph{characterizing fairness in recourse through the interplay between the starting qualifications of individuals and the amount of effort they need to exert over time to attain a favorable outcome}.

\begin{figure}[t!]
    \centering
    \includegraphics[width=\linewidth]{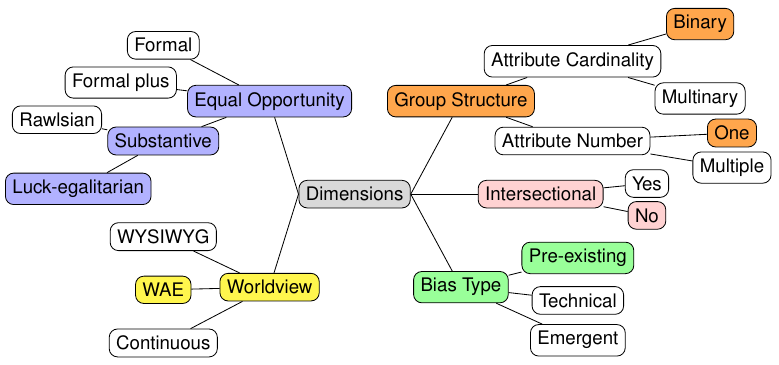}
    \caption{Summary of the normative dimensions of the fairness formalization and fairness-enhancing intervention used in this paper, according to the classification of \cite{zehlike2022fairness}.}
    \label{fig:mindmap}
\end{figure}

\paragraph*{Normative grounding.} Mitigating disadvantage due to morally-arbitrary factors and rewarding morally-relevant effort is consistent with the framing of equality of opportunity (EO) --- a philosophical doctrine that aims to remove morally irrelevant barriers to the attainment of desirable positions.  
Different conceptions of EO differ in the domains to which they apply~\cite{Fishkin2014Bottlenecks} and in how they model the relationship between morally relevant and irrelevant factors. We take the point of view of substantive and, specifically, of  luck-egalitarian EO~\cite{dworkin_1981, Roemer2002}, treating recourse over time as a series of contests to build qualifications and attain desirable positions.  
%

\cite{Roemer2002} proposed to consider people with the same morally-arbitrary circumstance as being of the same ``type.'' In the recourse setting, an individual's type informs \emph{both} their initial qualification score and the effort they exert when acting upon recourse.  Roemer posits that, since qualification and effort distributions are characteristics of the type, and since type membership is morally arbitrary, one should only compare individuals to others of the same type. Figure~\ref{fig:mindmap} summarizes the normative views that underlie the formalization of fairness in recourse in this paper. 
See~\cite{DBLP:conf/eaamo/KhanMS22} for additional background on EO, and for details about the normative grounding of algorithmic fairness objectives. 

\begin{figure}[t]
    \centering
    \includegraphics[width=0.95\linewidth]{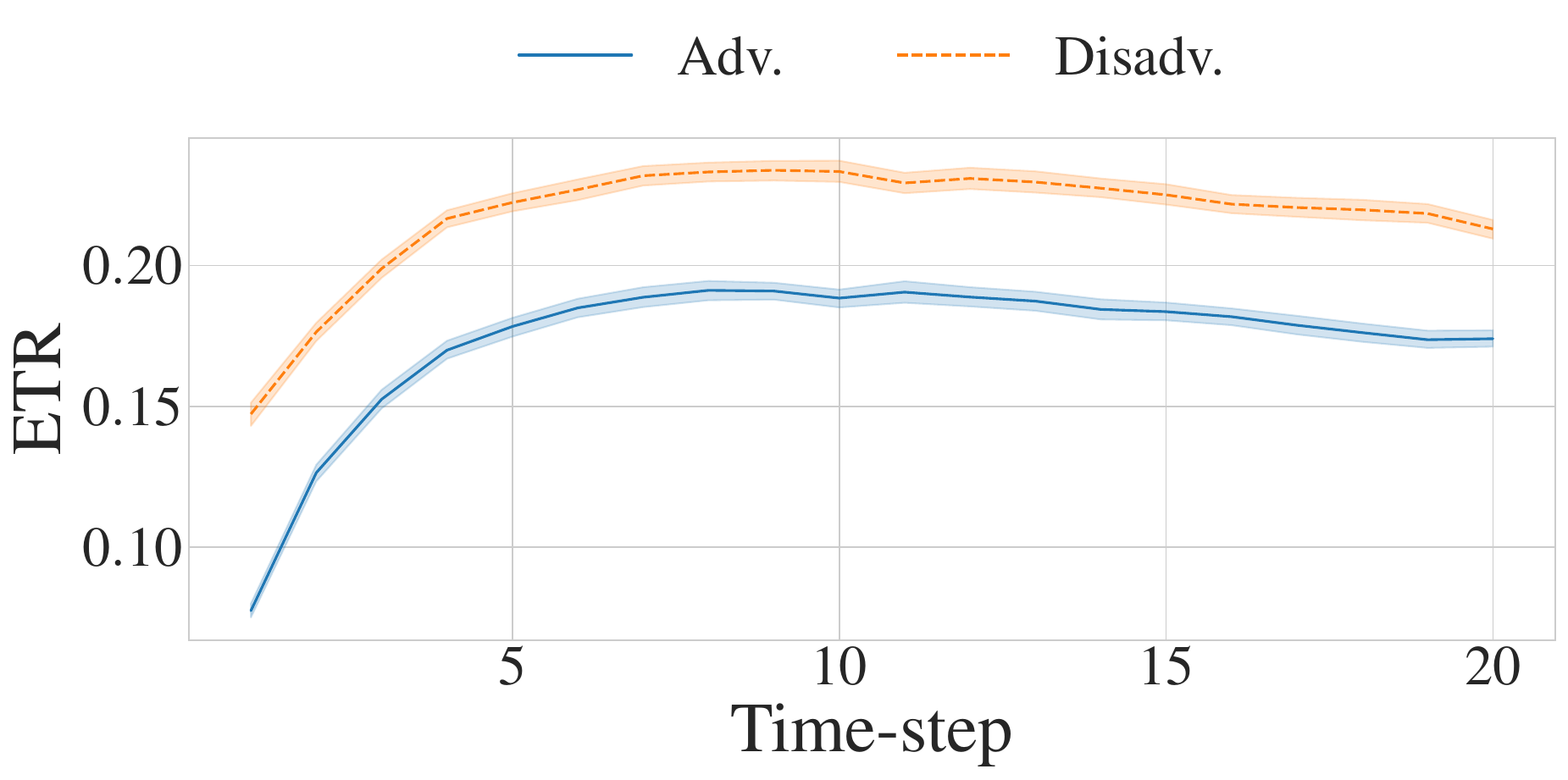}
    \caption{The \etr, or the total amount of effort exerted per successful recourse event, by advantaged and disadvantaged groups (in terms of initial circumstances) over 20 time-steps.} 
    \label{fig:motivation}
     \vspace{-4mm}
\end{figure}
\paragraph*{Contributions.}

The first contribution of this paper is introducing two novel metrics for evaluating the fairness of recourse. The first, \etr, quantifies the amount of effort it takes for individuals in one population to achieve recourse, compared to individuals in another population. This metric is illustrated in Figure~\ref{fig:motivation}, which shows that across all time-steps, a disadvantaged population exerts noticeably more effort per successful recourse event than an advantaged one. The second metric, \ttr, quantifies how long it takes for individuals in one population to achieve recourse, compared to individuals in another population.


The second contribution is extensive empirical analysis of the interplay between initial qualifications and effort, conducted with the help of a flexible agent-based simulation framework.  This analysis surfaces unfairness, in the sense of disparities in \etr and \ttr across population groups.

The third contribution is an intuitive fairness intervention, called \fairselect, that is consistent with the normative grounding of substantive EO.  We report results of an empirical analysis, comparing the behavior of our method to others, and showing that it can successfully mitigates disparities.

\section{Background}
\label{sec:related}

\paragraph{Fairness in recourse.}  There is an emerging body of work on fairness in algorithmic recourse. Notably, there is a \emph{disconnect between classifier fairness and fair recourse}---even with a fair classifier, no guarantees can be made about the fairness of recourse for that classifier. 

Figure~\ref{fig:data_generation} shows three different scenarios where decision-making may be fair with respect to the commonly used \emph{Demographic Parity} metric, but that results in a scenario where negatively classified instances from a disadvantaged population are much further away from the decision boundary than positively classified instances. As a result, members of the disadvantaged group would have to change their score by a much larger value to achieve recourse, which in practice means they would have to put in much more real-world effort. While there may be certain settings where this difference is appropriate, there are also many settings where it is undesirable (see the \textbf{Motivating example}).

This issue was studied by~\cite{gupta2019equalizing} who suggested measuring bias in recourse as the \emph{group-wise distance to the decision boundary}, and proposed a method for mitigating this bias, which we explore later in this paper. Recently, \cite{von2022fairness} added a causal lens to~\cite{gupta2019equalizing}, and defined two new metrics: one for individual-level and one for group-level fairness. These metrics are based on the \emph{cost} of recourse, \ie how much \emph{effort} individuals must exert to change their outcome.

Equalizing the distance to the decision boundary for recourse recommendations across groups may be insufficient because the ability to act on recourse recommendations may be different across groups, as noted by~\cite{ustun2019actionable}. In other words, the recourse mechanism \emph{itself} might be unfair. Suppose that individuals are described by two features, $X_1$ and $X_2$, but that $X_1$ is more difficult to change than $X_2$. Unfairness can arise if individuals from the disadvantaged group are more often being recommended to change $X_1$. This issue can be mitigated, at least in part, by creating similar group-wise counterfactual explanations, as in~\cite{artelt2022explain}.~\cite{sharma2019certifai} also explored differences in the cost---described as the \emph{burden}---of recourse between groups.

Our methods differ from prior work in our explicit modeling of qualification and effort, as well as an explicit focus on fairness and recourse dynamics over time---an oft-overlooked element in work on recourse, discussed next.

\paragraph{Robustness and time in recourse.} Too often literature in algorithmic recourse assumes a stable setting over time, which is unlikely in the real-world~\cite{barocas2020hidden}.  In light of this, there has been study in the \emph{robustness of recourse}, or the idea that recourse generated from a classifier should be stable to perturbations in the setting~\cite{guyomard2023generating,konig2021causal,mishra2021survey,sharma2019certifai,ehyaei2023robustness}.

\cite{ferrario2022robustness} studied the impact of retraining models on the validity of counterfactual explanations over time. They developed a method using counterfactual data augmentation to improve robustness. ~\cite{pawelczyk2022probabilistically} examined cases when an individual acts on recourse but does not reach the cut-off and so does not receive the positive outcome. Others have considered the competitive effects on recourse that result from individuals vying for scarce resources~\cite{altmeyer2023endogenous,o2021multi, o2022toward,rawal2020beyond,upadhyay2021towards}.

\paragraph*{Modeling multi-agent recourse over time.} In this work, we build on an approach for modeling multi-agent recourse over time by ~\cite{DBLP:conf/eaamo/FonsecaBABS23}. Consider a population of agents $P$. Each $x \in P$ is described by a set of features $x \in \mathcal{X}$, and who will be evaluated for a desired outcome by a machine learning classifier or ranker $f: \mathcal{X} \rightarrow [0,1]$ (\eg applying for an academic program). Agents are competing for a limited number of positive outcomes $k$ (\eg the number of spots available in the program) over a series of timesteps $t = \{0,1,...,T\}$. An agent may change its features over time, and we use $x_t$ to refer to the state of $x$ at time $t$.

At each timestep, a score $f(x_t)$ is calculated for each agent. This score is used to assign outcomes: positive when $f(x_t)\geq s_t$, where $s_t$ is the score cut-off at time $t$, and negative when $f(x_t)< s_t$.  Agents who receive a negative outcome also get
a \emph{recourse recommendation} $x' \in \mathcal{X}$ that satisfies two conditions: (1) $f(x') \geq s_t$ and (2) $x'$ is associated with the lowest cost $c(x,x')$ of making the change:
\begin{equation}
x' = \min_{x'} \quad c(x_t,x') \text{~~~~~s.t.} \quad  f(x') \geq s_t, x' \in \mathcal{X}
\label{eq:recourse}
\end{equation}

Agents who receive the positive outcome exit the simulation. A new set of agents $N_t$ enters the simulation at each time step. Figure~\ref{fig:simulation_visual} illustrates the dynamics of recourse from ~\cite{DBLP:conf/eaamo/FonsecaBABS23}.

\begin{figure}[]
    \centering
    \includegraphics[width=0.9\linewidth]{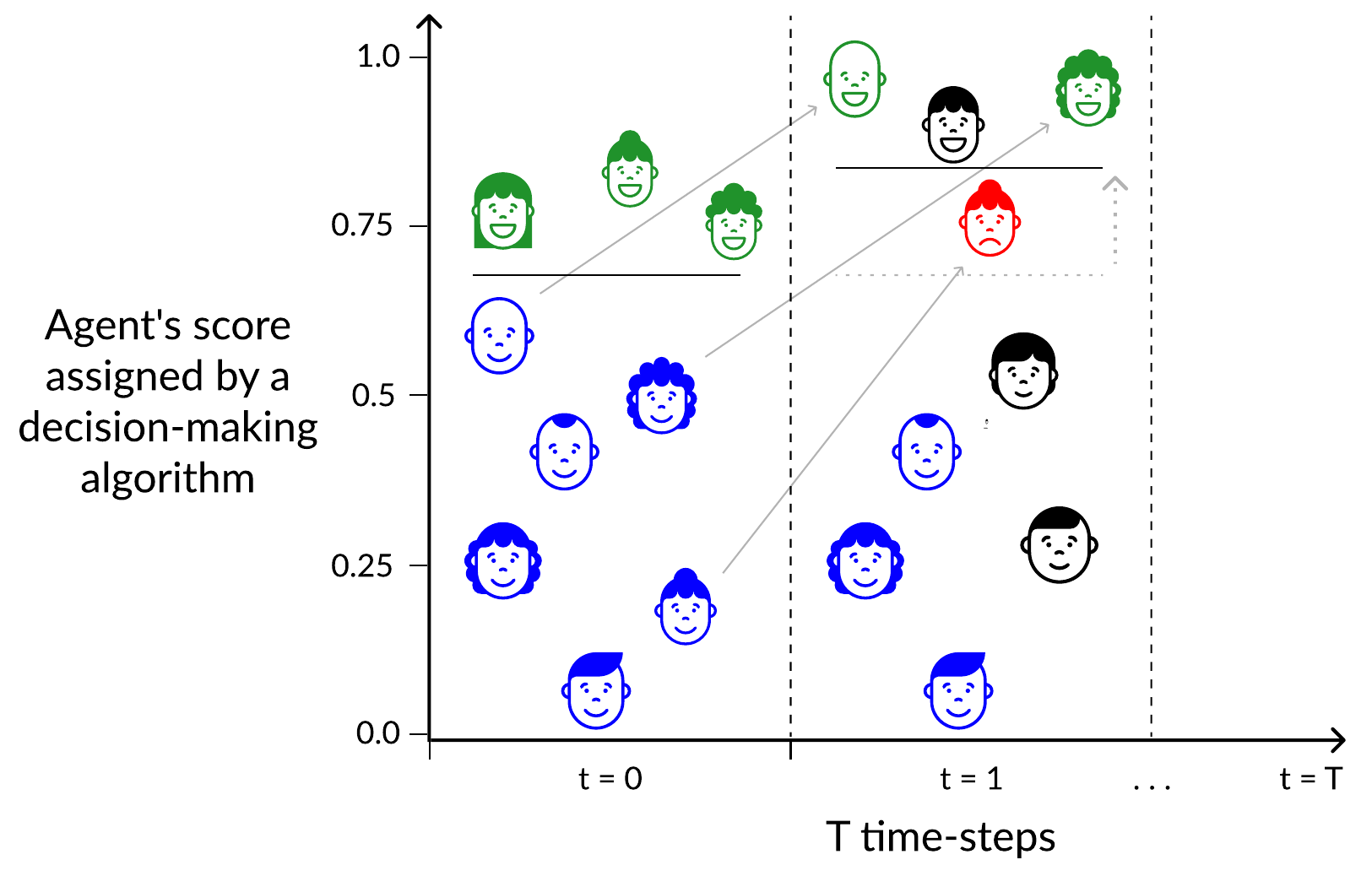}
    \caption{Reproduced with permission from~\cite{DBLP:conf/eaamo/FonsecaBABS23}.The x-axis shows time-steps $t$, and the y-axis shows agent scores  $f(x_t)$. In this example, there are $k = 3$ positive outcomes available at each time-step. At $t = 0$, green agents receive a positive outcome ($f(x_0) \geq s_0$, where $s_0$ is represented by the horizontal line), and blue agents receive a negative outcome along with a recommendation $x'$ on how to change their features to receive a positive outcome. At time $t = 1$, new agents $N_1$, shown in black, enter the environment. Grey arrows show recourse actions. The agent shown in red acted on the recourse recommendation as directed, but (disappointingly) its effort turned out to be insufficient because competition from other agents ``raised the bar'' for acceptance.
     }
    \label{fig:simulation_visual}
   \vspace{-4mm}
\end{figure}

To model the likelihood an agent $x_t$ will act on a recourse recommendation $x'$, and the amount by which they change their features, a function $a_e: \mathcal{X} \times \mathcal{X} \rightarrow \mathcal{X}$ is used. The function uses the hyperparameter $e \in \mathbb{R}$ that represents an agent's individual willingness to take action, or what will refer to as \textbf{\emph{effort}}. We discuss our choices for $e$ in Section~\ref{subsec:pops}.

\section{Fairness in Recourse}
\label{sec:metrics}

\begin{figure*}
     \centering
     \begin{subfigure}[t]{0.33\textwidth}
         \centering
         \includegraphics[width=\textwidth]{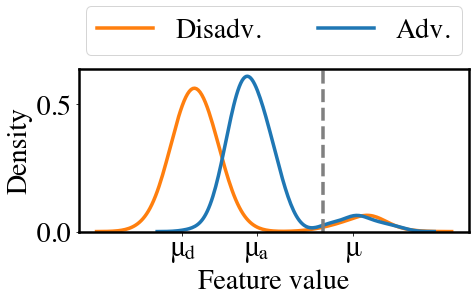}
         \caption{$\mu$ is the same for both groups; different means $\mu_a$ and $\mu_d$; equal variances}
     \end{subfigure}
     \begin{subfigure}[t]{0.33\textwidth}
         \centering
         \includegraphics[width=\textwidth]{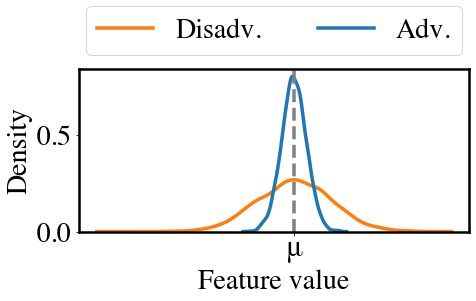}
         \caption{equal means, different variances}
     \end{subfigure}
     \begin{subfigure}[t]{0.33\textwidth}
         \centering
         \includegraphics[width=\textwidth]{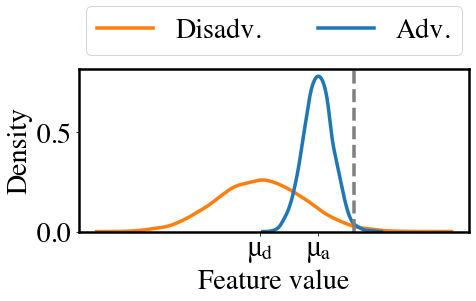}
         \caption{different means, different variances}
     \end{subfigure}

    \caption{Illustration of score distributions that can result in fair decision-making but unfair recourse. Let individuals to the right of the vertical dashed grey line receive a positive outcome; then in all cases, decision-making is fair with respect to \emph{Demographic Parity} between the advantaged (blue) and disadvantaged (orange) groups, but recourse is unfair. In our experiments, we assume that features are generated according to (a), where $\mu$ is the mean value for the high-performing agents from both populations, and $\mu_a$ and $\mu_d$ are the means of lower-performing agents of the advantaged and disadvantaged populations, respectively.
    }
    
    \label{fig:data_generation}
    \vspace{-4mm}         
\end{figure*}

\subsection{Populations, Initial Qualificaitons, and Effort}
\label{subsec:pops}
We assume that each individual belongs to one of several populations, one of which could map to a real-world protected group. In this paper, we restrict our attention to the case where agents are partitioned into two populations: an advantaged group $P^a$ and a disadvantaged group $P^d$, such that $P = P^a \cup P^d$ and $P^a \cap P^d = \emptyset$. In Figure~\ref{fig:mindmap}, this is shown under ``group structure,'' and corresponds to one binary protected attribute that denotes population membership. Importantly, each $P^a$ and $P^d$ have an associated \emph{initial qualification} and \emph{effort}, which we now discuss in turn.  

\emph{Initial qualification} refers to the initial circumstances of individuals. In our implementation, we use a qualification parameter $p$ that governs the way initial features values are drawn for individuals. Figure~\ref{fig:data_generation} shows different initial feature distributions for agents that can result in fair decision-making with respect to \emph{Demographic Parity}, but leads to \emph{unfair recourse}.   In this paper, we model case (a), where the features of advantaged $P^a$ and disadvantaged $P^d$ populations are drawn from a bimodal distribution, representing high-performing and lower-performing agents.  For both populations, features for high-performing agents are drawn from $\mathcal{N}(\mu,\sigma^2)$. For lower-performing agents, features are drawn from $\mathcal{N}(\mu_a,\sigma^2)$ and $\mathcal{N}(\mu_d,\sigma^2)$, respectively, where $\mu_d \leq \mu_a < \mu$. For simplicity, all distributions have the same standard deviation $\sigma$. We will control the disparity in initial qualifications between agents in $P^a$ and $P^d$ by controlling the distance between the means with parameter $q \in \mathbb{N}$, such that $\mu_a = \mu_d + q * \sigma$.

\emph{Normatively} (per Figure~\ref{fig:mindmap}), by introducing a disparity in initial qualifications between $P^a$ and $P^d$, we model pre-existing bias, per~\cite{DBLP:journals/tois/FriedmanN96}. The assumption that pre-existing bias is due to disparities in access to opportunities to build qualifications ahead of the competition (rather than to some ``innate'' ability) is consistent with the We are all Equal (WAE) worldview from~\cite{DBLP:journals/cacm/FriedlerSV21}.  Our choice to focus on the feature generation method from Figure~\ref{fig:data_generation}(a), where populations have similar performance distributions (equal variance) but many members of the disadvantaged group have lower initial qualifications (different means), is consistent with this normative framing.

\emph{Effort} refers to the improvement in an agent's score between rounds. In this paper, \emph{effort} is drawn from a folded normal distribution that corresponds to $\mathcal{N}(e,1)$. Therefore, the larger $e$, the higher the potential improvement. \emph{Normatively}, we make the choice to associate an effort distribution with each population, which is consistent with luck-egalitarian substantive EO (see Figure~\ref{fig:mindmap}).  To represent and control the potential disparity in effort between $P^a$ and $P^d$, we consider three cases:  (1) effort sampled from the same distribution for $P^a$ and $P^d$ (\emph{i.e.}, $e_a = e_d$), (2) effort of $P^a$ is in expectation double that of $P^d$ (\emph{i.e.}, $e_a > e_d$) and (3) effort of $P^a$ is in expectation half that of $P^d$ (\emph{i.e.}, $e_a < e_d$). 

\subsection{Metrics}

\paragraph*{Effort-to-recourse (ETR).}

This metric quantifies the amount of effort exerted by individuals who acted on recourse \emph{and subsequently received a positive outcome}, and compares this effort across populations.

Recall that $s_i$ is the threshold for the positive outcome at time $i$. As the first step, to calculate $ETR$ for population $P$ at time $t$, we identify the \emph{set of recourse-successful agents $R_t$} who had a negative outcome at some time $i<t$, acted on recourse, and then received a positive outcome at time $t$:

\begin{equation}
    D_t (P)= \{x \in P ~|~ f(x_t) \geq s_t \land \exists~i<t: f(x_i)< s_i\}
\end{equation}

Next, we define the \emph{effort of agent $x$} up to time $t$ as:

\begin{equation}\label{eq:cost}
C_t(x) = \sum_{\forall i < t} c(x_{i},x_{i+1})
\end{equation}

Using these two elements, we express \etr:
\begin{equation}\label{eq:etr}
    ETR_{t}(P) = \frac{1}{|D_t (P)|} \sum_{x \in D_t (P)} C_t(x)
\end{equation}

We can compute \etr for the overall population $P$ or separately for $P^a$ and $P^d$. Further, we can compute \etr at each time step, and use it to track effort over time, as illustrated in Figure~\ref{fig:motivation}. Finally, we can compute \etr values for sub-populations, and compare them to reason about fairness.  In this paper, we quantify \emph{effort-to-recourse disparity} as the ratio:

\begin{equation}
rETR_{t}(P^a, P^d) = \frac{ETR_{t}(P^d)}{ETR_{t}(P^a)}
\label{eq:etr_ratio}
\end{equation}

When $rETR_{t}$  is close to $1.0$, then agents from populations $P^a$ and $P^d$ exert similar effort per successful recourse event. However, when $rETR_{t}>1$, then the disadvantaged agents from $P^d$ exert more effort than their counterparts from $P_a$.  To reason about the magnitude of the disparity, we may follow the guideline used by practitioners in algorithmic fairness known as the ``80\%-rule,'' where per-group values should be within $80\%$ of each other, placing legally permissible disparity values within the $[0.8,1.2]$ range~\cite{bell2023possibility}. 

Finally, we note that the $rETR_t$ metric has a convenient mathematical property: its value is proportional to the amount of expected effort exerted by members of each population: $rETR_{t}(P^a, P^d) \propto \frac{e_d}{e_a}$. See formal statement and proof of this property in Proposition 1 in the Appendix.

\paragraph{Time-to-recourse (TTR).} This metric is defined analogously to $ETR$: it quantifies the average number of time steps per successful recourse event up to a given time $t$.  

For a recourse-successful agent $x \in R_t(P)$, let $\delta(x)$ denote the number of time-steps to achieve a positive outcome. More formally, we can write $o^{-}(x) = \min \{i\leq t: f(x_i) < s_i\}$ as the first time step when the agent received a negative outcome,
$o^{+}(x) = \{i\leq t: f(x_i) \geq s_i\}$ as the timestep when the agent received a positive outcome,
and, finally, $\delta(x) = o^{+}(x) - o^{-}(x)$.
We express \emph{time-to-recourse} as:

\begin{equation}
    TTR_{t}(P) = \frac{1}{|R_t (P)|} \sum_{x \in R_t (P)} \delta(x)
    \label{eq:ttr}
\end{equation}

We quantify \emph{time-to-recourse disparity} as the difference:

\begin{equation}
dTTR_{t}(P^a, P^d) = TTR_{t}(P^d) - TTR_{t}(P^a)
\label{eq:ttr_diff}
\end{equation}

Note that we use a difference instead of a ratio with $TTR$ because it is more comprehensible: $dTTR_{t} $ is the literal difference in the number of time-steps between sub-populations. A value of 0 means that there is no difference in the time to a positive outcome between groups. Positive values represent \emph{how much more time on average} it takes an agent from the disadvantaged population $P^d$ to achieve the positive outcome (negative values imply the opposite).

Finally, note that both $rETR_{t}$ and $dTTR_{t}$ only account for productive effort (resp. time), because they are computed for the set of recourse-successful agents $R_t$. In experiments that measured both productive and ``wasted effort'' (\ie the effort of individuals who \emph{never} received a positive outcome), we found that there was no difference in our findings.
\section{Mitigating Unfairness}
\label{sec:mitigiation}

Disparities in $ETR$ and $TTR$ can be mitigated by modifying the decision-making process. For instance, one could pre-suppose a fixed number (or proportion) of applicants of each population to accept. Through this type of strategy, individuals are ``competing'' within population rather then across populations.  Because both initial circumstance and effort are associated with a population, such \emph{circumstance-normalized selection} strategy can be considered fair under substantive EO.

Following this justification, we develop an intuitive selection intervention, based on \emph{rank-aware proportional representation}  by~\citet{yang2017measuring}.  It involves assigning positive outcomes to the highest-scoring individuals from each sub-population, proportionally by population size.\footnote{\citet{yang2017measuring} developed this approach for ranking, while our setting is \emph{set selection}, a special case of ranking.} When $k$ is even and sub-populations are of equal size, then an equal number of positive outcomes is assigned to individuals from each sub-population. We refer to this strategy as \fairselect (\cns for short).

Since the \cns approach is post-hoc (\ie after scoring and generating recourse recommendations), it has the advantage of being agnostic to both the method used to generate individual scores, and the recourse recommendation method. Note that, in our implementation, ties are broken by assigning the positive outcome to the agent who has been trying for recourse the longest, and arbitrarily after that.

\paragraph*{Benchmark strategy.} To better understand the effectiveness of our proposed method, \cns, in mitigating disparities in $ETR$ and $TTR$, we compare its performance to an existing pre-processing strategy for improving the robustness of recourse called \emph{Counterfactual Data Augmentation} (or \cda) by~\cite{ferrario2022robustness}. \cda augments the data with counterfactuals for individuals who received the negative outcome, and then re-trains the classifier on this data. By design, \cda  is most effective when re-training at each time-step.

\paragraph*{Other intervention strategies.} We also implemented the in-processing method by ~\citet{gupta2019equalizing}, who re-position the decision boundary of a classifier to be equidistant from negatively-classified individuals from different groups. As a result, differences in initial circumstances (although not in effort) are mitigated. We found that this method had an unintended negative consequence: in one experiment, the decision boundary slowly shifted its position so as to be completely inverse, meaning it encouraged individuals to \emph{worsen} their feature values to achieve recourse (illustrated in Figure~\ref{fig:boundary_shift_one}). For this reason, we do not include the method as a benchmark, but do present these results in Appendix Table~\ref{tab:suppl_results}.

As an alternative system-level intervention, schools or other institutions can design recourse for systems with \emph{effort explicitly in mind}. For example, returning school or loan applicants could be evaluated solely on the magnitude of improvement in their score. This intervention may also be useful in a setting where individuals are appealing social media bans, and need to show that they have ``learned from the consequences'' of their behavior.
\section{Experimental Evaluation}
\label{sec:experiments}

\begin{table*}[t]
\small
\centering
\caption{The effort-to-recourse disparity ($rETR$) and disparate time-to-recourse ($dTTR$) between an advantaged population $P_a$ (used as reference) and a disadvantaged population $P_d$ under different initial circumstances as described by the qualification $q$ and the 3 effort conditions defined in Section~\ref{subsec:pops}. Results are shown for a baseline (no mitigation), with \cns, and with \cda. Bold indicates the lowest (\ie best) $rETR$ or $dTTR$ for a setting. All disparities are calculated cumulatively from the initial time-step to $T$.}
\label{tab:results}

\begin{tabular}{cc|cccc|ccc}   
\textbf{} & \textbf{} & \multicolumn{4}{c|}{\textbf{$rETR$}} & \multicolumn{3}{c}{\textbf{$dTTR$}} \\
\textbf{} & \textbf{q} & \textbf{\begin{tabular}[c]{@{}c@{}}Baseline\end{tabular}} & \textbf{\begin{tabular}[c]{@{}c@{}}\cns\end{tabular}} & \textbf{\cda} & \textbf{\cns + \cda} & \textbf{\begin{tabular}[c]{@{}c@{}}Baseline\end{tabular}} & \textbf{\begin{tabular}[c]{@{}c@{}}\cns\end{tabular}} & \textbf{\cda} \\ \hline
 \parbox[t]{2mm}{\multirow{4}{*}{\rotatebox[origin=c]{90}{$e_a = e_d$}}} & 0 & 1.001 $\pm$ 0.003 & 1.006 $\pm$ 0.004 & 1.001 $\pm$ 0.001 & 0.999 $\pm$ 0.001 & -0.043 $\pm$ 0.022 & 0.006 $\pm$ 0.014 & -0.026 $\pm$ 0.021 \\
 & 1 & 1.213 $\pm$ 0.003 & 1.066 $\pm$ 0.004 & 1.079 $\pm$ 0.001 & \textbf{1.044 $\pm$ 0.002} & 0.646 $\pm$ 0.023 & \textbf{0.021 $\pm$ 0.013} & 1.012 $\pm$ 0.021 \\
 & 2 & 1.429 $\pm$ 0.004 & 1.133 $\pm$ 0.004 & 1.158 $\pm$ 0.002 & \textbf{1.020 $\pm$ 0.005} & 1.177 $\pm$ 0.024 & \textbf{0.022 $\pm$ 0.013} & 1.877 $\pm$ 0.022 \\
 & 3 & 1.640 $\pm$ 0.006 & 1.207 $\pm$ 0.005 & 1.237 $\pm$ 0.002 & \textbf{1.063 $\pm$ 0.004} & 1.573 $\pm$ 0.024 & \textbf{0.038 $\pm$ 0.013} & 2.477 $\pm$ 0.022 \\ \hline
\parbox[t]{2mm}{\multirow{4}{*}{\rotatebox[origin=c]{90}{$e_a > e_d$}}} & 0 & 0.942 $\pm$ 0.002 & 0.737 $\pm$ 0.002 & 0.994 $\pm$ 0.001 & 0.940 $\pm$ 0.001 & 1.315 $\pm$ 0.021 & 0.100 $\pm$ 0.012 & 1.290 $\pm$ 0.021 \\
 & 1 & 1.109 $\pm$ 0.003 & 0.805 $\pm$ 0.002 & \textbf{1.049 $\pm$ 0.001} & 0.916 $\pm$ 0.004 & 1.835 $\pm$ 0.023 & \textbf{0.084 $\pm$ 0.012} & 2.172 $\pm$ 0.021 \\
 & 2 & 1.275 $\pm$ 0.003 & 0.886 $\pm$ 0.002 & 1.108 $\pm$ 0.001 & \textbf{0.901 $\pm$ 0.005} & 2.257 $\pm$ 0.023 & \textbf{0.050 $\pm$ 0.015} & 2.851 $\pm$ 0.022 \\
 & 3 & 1.429 $\pm$ 0.004 & 0.975 $\pm$ 0.003 & 1.170 $\pm$ 0.002 & \textbf{0.951 $\pm$ 0.003} & 2.532 $\pm$ 0.023 & \textbf{-0.029 $\pm$ 0.015} & 3.295 $\pm$ 0.025 \\ \hline
\parbox[t]{2mm}{\multirow{4}{*}{\rotatebox[origin=c]{90}{$e_a < e_d$}}} & 0 & 1.064 $\pm$ 0.002 & 1.366 $\pm$ 0.004 & 1.007 $\pm$ 0.001 & 1.064 $\pm$ 0.002 & -1.338 $\pm$ 0.021 & -0.101 $\pm$ 0.014 & -1.292 $\pm$ 0.018 \\
 & 1 & 1.252 $\pm$ 0.003 & 1.477 $\pm$ 0.005 & \textbf{1.064 $\pm$ 0.001} & 1.090 $\pm$ 0.001 & -0.845 $\pm$ 0.021 & \textbf{-0.085 $\pm$ 0.014} & -0.426 $\pm$ 0.019 \\
 & 2 & 1.433 $\pm$ 0.004 & 1.576 $\pm$ 0.006 & \textbf{1.123 $\pm$ 0.002} & 1.127 $\pm$ 0.002 & -0.446 $\pm$ 0.024 & \textbf{-0.065 $\pm$ 0.014} & 0.335 $\pm$ 0.019 \\
 & 3 & 1.605 $\pm$ 0.005 & 1.672 $\pm$ 0.006 & 1.190 $\pm$ 0.003 & \textbf{1.175 $\pm$ 0.003} & -0.168 $\pm$ 0.024 & \textbf{-0.047 $\pm$ 0.015} & 0.925 $\pm$ 0.022
\end{tabular}

\end{table*}

We ran each agent-based simulation over $T=20$ time-steps. To generate the score for each agent, we use a simple ranker $f(x) = 0.5 \cross x_1 + 0.5 \cross x_2$, where $x_1$ and $x_2$ are generated features.  Each reported metric represents the average value over 100 simulation runs, with different random seeds. In each simulation, there were 1,000 initial agents, 100 positive outcomes available ($k$), and 100 new agents at each time-step. Agents adapted according to the \emph{continuous, flexible} setting, see~\citet{DBLP:conf/eaamo/FonsecaBABS23} for details.

Our main experimental results are reported in Table~\ref{tab:results}. The table shows the \etr disparity ratio ($rETR$) and the \ttr difference, measured in time-steps, ($dTTR$) between an advantaged and disadvantaged populations ($P^a$ and $P^d$, respectively) under 4 different initial qualification conditions $q$, and under 3 different effort conditions $e$. Recall from Section~\ref{subsec:pops} (and Figure~\ref{fig:data_generation}(a)) that $q$ controls the disparity in initial qualifications between agents in $P^a$ and $P^d$: the distance between the means of the lower-performing agents is $\mu_a = \mu_d + q * \sigma$.

Recall also that, when the means of the folded normal effort distributions $e_a$ and $e_d$ are equal, then both populations exert comparable effort; $e_a > e_d$ denotes that the expected value of effort for agents in $P^a$ is twice that of $P^d$, and vice-versa for $e_a < e_d$. We report results for no mitigation (baseline), \cns (our mitigation), and \cda (benchmark mitigation, due to~\cite{ferrario2022robustness}).

\emph{Importantly, the ranker in our experiments results in fair selection w.r.t the commonly used Demographic Parity (DP) metric}: DP is close to $1.0$ for all values of $q$,  implying that our selection approach would colloquially be considered fair.

\paragraph{The effect of differences in initial qualification and effort.} When $q=0$ (\ie there is no difference in the initial qualifications between groups) $rETR$ is close to 1.0 under all effort conditions, and $dTTR$ is small. As expected, when $q$ increases, so does unfairness in $rETR$ and $dTTR$. Recall when $rETR$ is outside of the range $[0.8,1.2]$, we consider it unfair, and $rETR$ values over $1.0$ indicate how many times more effort $P^d$ exerts per successful recourse-event than $P^a$.

Disparities also grow based on the effort conditions. Fixing $q$, disparities are lowest when the $P^d$ exerts more effort, in the middle when expected effort between populations is equal, and highest when $P^a$ has higher expected effort.
The conditions with the highest observed disparities are predictably when $q=3$ and when $e_a > e_d$ ($P^a$ has better circumstances and has a higher expected effort). In this case, without mitigation, the effort per successful recourse event exerted by members of the disadvantaged population is 1.429 $\pm$ 0.004 times more than the advantaged population, and it takes them on average 2.532 $\pm$ 0.023 more time-steps to achieve recourse.

\paragraph{How much effort is enough?} Recall that the condition $e_a < e_d$ in Table 1 means that the disadvantaged population is, in expectation, doing twice as much ``hard work'' as the advantaged population. As expected, under this setting, agents in the disadvantaged group are putting in more effort, regardless of $q$---but importantly, this increased effort does not result in significant gains in $dTTR$. When $q = 1$ and $q=2$, agents in $P^d$ exert 1.433 $\pm$ 0.004 and 1.605 $\pm$ 0.005 times more effort, respectively, than agents in $P^a$, yet, they achieve those successful recourse events at nearly the same time as agents in $P^a$ (within a half time-step). This occurs because agents in $P^d$ \emph{will always require more time to overcome initial qualification differences}, and, as a result, are continually subject to competitive effects both between agents in the environment and with new agents entering at each time-step. Taken together, these results show that \emph{effort alone may not be enough to overcome disparities in initial circumstances across populations}. 

\paragraph{Mitigating disparities in ETR and TTR.} When $e_a = e_b$ and $e_a > e_b$, \cns is very effective at mitigating unfairness. In all cases, $rETR$ is close to $1.0$ and within tolerable ranges, and in many cases $dTTR$ is slightly negative, indicating that the disadvantaged population is taking \emph{slightly less} time per successful recourse event. When $e_a < e_d$, \cns is slightly less effective at mitigating disparities in $ETR$---but still shows improvement over the baseline unmitigated case. Importantly, $dTTR$ is significantly lower in all cases.

\cns outperforms \cda at mitigating disparities in $dTTR$ in all settings. Notably, with respect to $rETR$, \cns performs similarly to the benchmark, and results vary by setting. \cda seems to be better at mitigating disparities when  $e_a < e_d$.

Notably, the benchmark strategy (\cda) is generally ineffective at mitigating disparities in $dTTR$, and is in many cases \emph{worse for individuals as compared to when no mitigation strategy is implemented at all}. While exploring what causes this is beyond the scope of this paper, we hypothesize that \cda is able to mitigate disparities in effort because it moves the decision boundary closer to the negative class. This results in less effort being required to cross the decision boundary, which has the downstream effect of more agents crossing the decision boundary, thereby increasing the competitiveness of the environment.

Our results show that the choice of mitigation strategy depends both on the fairness metric of interest ($rETR$ or $dTTR$) and on the setting (\ie the relationship between $e_a$ and $e_d$). Note that, since \cda is a pre- and in-processing intervention, and \cns is a post-processing intervention, they can be combined to create a strategy that works better than either of them individually. To this end, we ran the same experiment described in this section combining \emph{both} \cns and \cda. As expected, this combination of methods is highly effective at mitigating disparities on both $ETR$ and $TTR$, and is our recommended approach to mitigating unfairness in recourse. \cns + \cda is either the top performer for $rETRA$, or is within 0.05 of \cda.  For $dTTR$ (see Table~\ref{tab:suppl_results} in the Appendix) \cda and \cda + \cns are equally effective.

\begin{figure}[t!]
    \centering
    \includegraphics[width=0.45\textwidth]{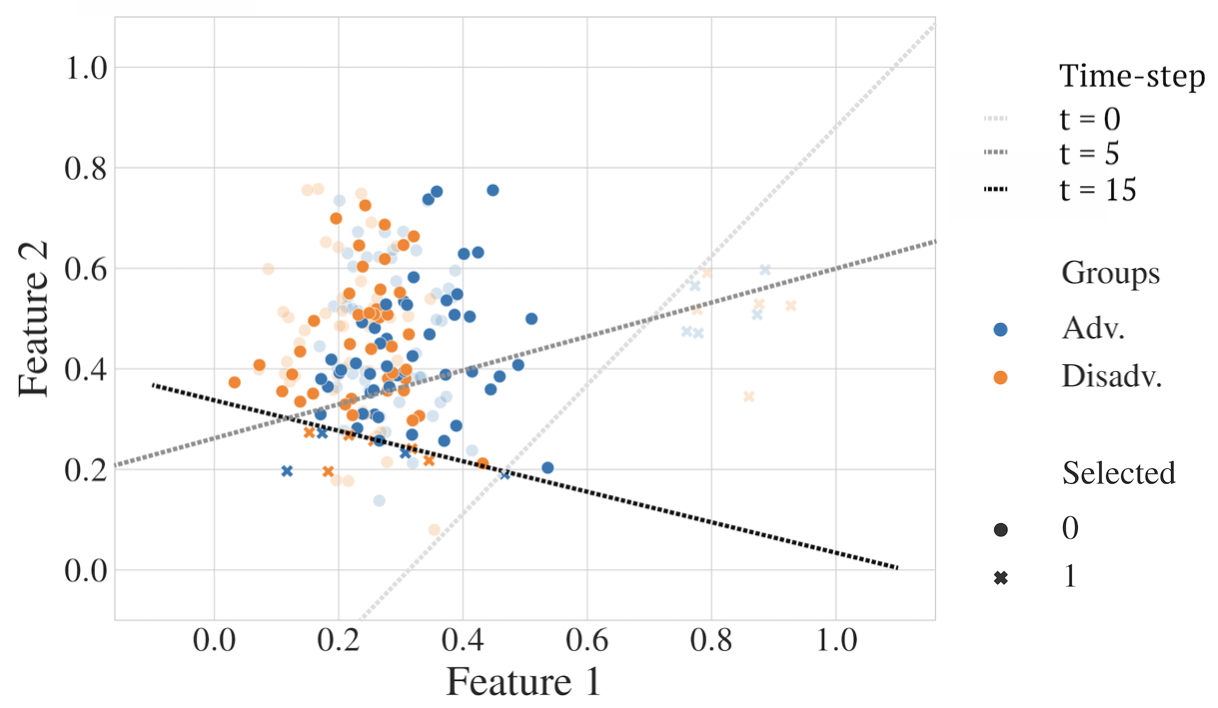}


     \caption{The figure shows agents' feature values and the decision boundary for a positive outcome (the dashed lines) for 3 time-steps from a single run of the simulation, using the in-processing Group Recourse Regularization method by ~\cite{gupta2019equalizing}. Note that faded-color points show agents' feature values at $t=0$, while full-color points show feature values at $t=15$.}
    \label{fig:boundary_shift_one}
    
\end{figure}
\section{Discussion}
\label{sec:discussion}

{\bf The impact of time.} This work is part of a growing body of research showing that it is critical to account for \emph{time-related effects} in algorithmic recourse (\cite{ferrario2022robustness,DBLP:conf/eaamo/FonsecaBABS23}). While disparities in the amount of effort exerted per successful recourse event are causes for concern in-and-of themselves, they become particularly salient when considered in the context of time. In our experiments, we observed that it is  more difficult for disadvantaged populations to successfully act on recourse---\emph{taking more time} and \emph{more repeated effort}---even under a classifier that is fair. This is problematic because it means that disadvantaged populations will be affected disproportionally by time-related competitive effects, and  data and model drift (\cite{altmeyer2023endogenous}). In essence, we find yet another manifestation of the ``poor get poorer''---or, in this context, the ``marginalized become more marginalized''--- showing that disparities in initial qualifications may be exacerbated over time.

We offer the following three guiding questions to practitioners designing fair algorithmic systems: (1) are recourse recommendations being generated fairly across groups (\eg is one population receiving easier recommendations)? (2) are populations exerting similar effort per successful recourse event (\eg is there a disparity on effort-to-recourse)? and (3) how problematic are time-related effects in this context, especially for individuals who are taking longer to act on a recourse recommendation and a re-enter the competition?

Further, unless practitioners are only concerned with mitigating disparities on $TTR$, we recommend that they take a combined approach and use both \cns and \cda. If the only metric of concern is $TTR$, it is likely that \cns is sufficient. If the metric of concern is $ETR$, the correct mitigation strategy is less obvious, and may depend on the way agents are exerting effort over time.

\paragraph{Trade-off between effort-to-recourse and overall utility.} The effort-to-recourse fairness metric is based on the world-view that \emph{effort towards a goal} is at times more important than meeting a \emph{fixed threshold for that goal}. This means that when mitigating bias in effort-to-recourse, one must make decisions about positive outcomes not just based on an individuals' score, but also on the hard-work and effort they have made to earn that score. Importantly, this implies that at times individuals with ``sub-optimal'' scores relative to a known goal threshold are selected for a positive outcome, resulting in a trade-off between mitigating effort-to-recourse and the overall utility. We do not explore this trade-off for reasons of scope, and leave it for future work.

\section{Conclusion, Limitations, and Social Impact}
\label{sec:conclusion}

In this paper, we propose new notions of fairness in algorithmic recourse that practitioners can add to their existing toolbox of metrics. Importantly, we take the normative worldview of substantive equality of opportunity, which is was not represented in the recourse literature. We show the use of these metrics in an agent-based simulation.  We also present an intervention for mitigating found disparities, which performs better than an existing benchmark strategy, and where a combination of both strategies does best. We hope that our work will be used to design fairer systems for recourse.

\paragraph{Main limitation: lack of real-world examples and data for recourse research.} The entirety of well-cited work in algorithmic recourse relies on popular open and archival datasets that were not created with algorithmic recourse as their primary purpose, \ie Adult (\cite{misc_adult_2}), German Credit (\cite{german_credit}), and COMPAS (\cite{angwin2022machine}), or use synthetic (or semi-synthetic data). This limitation has been noted by others (\citet{amarasinghe2023explainable, ferrario2022robustness}). 

To our knowledge, there are no public real-world algorithmic recourse datasets. In fact, even the near-universally accepted example of ``recourse when applying for a bank loan`` does not seem to exist outside of literature. The closest real-world example may be the use of \emph{Adverse Action Codes} (or \emph{Reason Codes}\footnote{\url{https://reasoncode.org/}}), per the U.S. Fair Credit Reporting Act, are given to individuals whose applications for credit are denied. Notably, these codes are non-specific. For example, \emph{Code 24} states as a reason for denial: ``too many bankcards with high balance compared to credit limit.'' Contrary to assumptions made in algorithmic recourse literature, it does not specify \emph{how many} bankcards is ``too many,'' and \emph{how much} higher the applicants' balance is compared to their credit limit.

While we believe there is still meaningful value in this work (and in the great work by other authors in this area), we see this limitation as highly problematic for the field and make a call-to-action for finding and deploying better data sets and use-cases for algorithmic recourse.

\newpage 
\bibliographystyle{named}
\bibliography{ijcai24}

\begin{thebibliography}{}

\bibitem[\protect\citeauthoryear{Altmeyer \bgroup \em et al.\egroup }{2023}]{altmeyer2023endogenous}
Patrick Altmeyer, Giovan Angela, Aleksander Buszydlik, Karol Dobiczek, Arie van Deursen, and Cynthia Liem.
\newblock Endogenous macrodynamics in algorithmic recourse.
\newblock In {\em First IEEE Conference on Secure and Trustworthy Machine Learning}, 2023.

\bibitem[\protect\citeauthoryear{Amarasinghe \bgroup \em et al.\egroup }{2023}]{amarasinghe2023explainable}
Kasun Amarasinghe, Kit~T Rodolfa, Hemank Lamba, and Rayid Ghani.
\newblock Explainable machine learning for public policy: Use cases, gaps, and research directions.
\newblock {\em Data \& Policy}, 5:e5, 2023.

\bibitem[\protect\citeauthoryear{Angwin \bgroup \em et al.\egroup }{2022}]{angwin2022machine}
Julia Angwin, Jeff Larson, Surya Mattu, and Lauren Kirchner.
\newblock Machine bias.
\newblock In {\em Ethics of data and analytics}, pages 254--264. Auerbach Publications, 2022.

\bibitem[\protect\citeauthoryear{{Arif Khan} \bgroup \em et al.\egroup }{2022}]{DBLP:conf/eaamo/KhanMS22}
Falaah {Arif Khan}, Eleni Manis, and Julia Stoyanovich.
\newblock Towards substantive conceptions of algorithmic fairness: Normative guidance from equal opportunity doctrines.
\newblock In {\em Equity and Access in Algorithms, Mechanisms, and Optimization, {EAAMO} 2022, Arlington, VA, USA, October 6-9, 2022}, pages 18:1--18:10. {ACM}, 2022.

\bibitem[\protect\citeauthoryear{Artelt and Hammer}{2022}]{artelt2022explain}
Andr{\'e} Artelt and Barbara Hammer.
\newblock " explain it in the same way!"--model-agnostic group fairness of counterfactual explanations.
\newblock {\em arXiv preprint arXiv:2211.14858}, 2022.

\bibitem[\protect\citeauthoryear{Barocas \bgroup \em et al.\egroup }{2020}]{barocas2020hidden}
Solon Barocas, Andrew~D Selbst, and Manish Raghavan.
\newblock The hidden assumptions behind counterfactual explanations and principal reasons.
\newblock In {\em Proceedings of the 2020 conference on fairness, accountability, and transparency}, pages 80--89, 2020.

\bibitem[\protect\citeauthoryear{Becker and Kohavi}{1996}]{misc_adult_2}
Barry Becker and Ronny Kohavi.
\newblock {Adult}.
\newblock UCI Machine Learning Repository, 1996.
\newblock {DOI}: https://doi.org/10.24432/C5XW20.

\bibitem[\protect\citeauthoryear{Bell \bgroup \em et al.\egroup }{2023}]{bell2023possibility}
Andrew Bell, Lucius Bynum, Nazarii Drushchak, Tetiana Zakharchenko, Lucas Rosenblatt, and Julia Stoyanovich.
\newblock The possibility of fairness: Revisiting the impossibility theorem in practice.
\newblock In {\em Proceedings of the 2023 ACM Conference on Fairness, Accountability, and Transparency}, pages 400--422, 2023.

\bibitem[\protect\citeauthoryear{Chouldechova}{2017}]{DBLP:journals/bigdata/Chouldechova17}
Alexandra Chouldechova.
\newblock Fair prediction with disparate impact: {A} study of bias in recidivism prediction instruments.
\newblock {\em Big Data}, 5(2):153--163, 2017.

\bibitem[\protect\citeauthoryear{Dwork \bgroup \em et al.\egroup }{2012}]{DBLP:conf/innovations/DworkHPRZ12}
Cynthia Dwork, Moritz Hardt, Toniann Pitassi, Omer Reingold, and Richard~S. Zemel.
\newblock Fairness through awareness.
\newblock In Shafi Goldwasser, editor, {\em Innovations in Theoretical Computer Science 2012, Cambridge, MA, USA, January 8-10, 2012}, pages 214--226. {ACM}, 2012.

\bibitem[\protect\citeauthoryear{Dworkin}{1981}]{dworkin_1981}
Ronald Dworkin.
\newblock What is equality? part 1: Equality of welfare.
\newblock {\em Philosophy and Public Affairs}, 10(3):185--246, 1981.

\bibitem[\protect\citeauthoryear{Ehyaei \bgroup \em et al.\egroup }{2023}]{ehyaei2023robustness}
Ahmad-Reza Ehyaei, Amir-Hossein Karimi, Bernhard Sch{\"o}lkopf, and Setareh Maghsudi.
\newblock Robustness implies fairness in causal algorithmic recourse.
\newblock In {\em Proceedings of the 2023 ACM Conference on Fairness, Accountability, and Transparency}, pages 984--1001, 2023.

\bibitem[\protect\citeauthoryear{Feldman \bgroup \em et al.\egroup }{2015}]{feldman2015certifying}
Michael Feldman, Sorelle~A Friedler, John Moeller, Carlos Scheidegger, and Suresh Venkatasubramanian.
\newblock Certifying and removing disparate impact.
\newblock In {\em proceedings of the 21th ACM SIGKDD international conference on knowledge discovery and data mining}, pages 259--268, 2015.

\bibitem[\protect\citeauthoryear{Ferrario and Loi}{2022}]{ferrario2022robustness}
Andrea Ferrario and Michele Loi.
\newblock The robustness of counterfactual explanations over time.
\newblock {\em IEEE Access}, 10:82736--82750, 2022.

\bibitem[\protect\citeauthoryear{Fishkin}{2014}]{Fishkin2014Bottlenecks}
Joseph Fishkin.
\newblock {\em Bottlenecks: A New Theory of Equal Opportunity}.
\newblock Oup Usa, 2014.

\bibitem[\protect\citeauthoryear{Fonseca \bgroup \em et al.\egroup }{2023}]{DBLP:conf/eaamo/FonsecaBABS23}
Jo\~{a}o Fonseca, Andrew Bell, Carlo Abrate, Francesco Bonchi, and Julia Stoyanovich.
\newblock Setting the right expectations: Algorithmic recourse over time.
\newblock In {\em Proceedings of the 3rd ACM Conference on Equity and Access in Algorithms, Mechanisms, and Optimization}, EAAMO '23, New York, NY, USA, 2023. Association for Computing Machinery.

\bibitem[\protect\citeauthoryear{Friedler \bgroup \em et al.\egroup }{2021}]{DBLP:journals/cacm/FriedlerSV21}
Sorelle~A. Friedler, Carlos Scheidegger, and Suresh Venkatasubramanian.
\newblock The (im)possibility of fairness: different value systems require different mechanisms for fair decision making.
\newblock {\em Commun. {ACM}}, 64(4):136--143, 2021.

\bibitem[\protect\citeauthoryear{Friedman and Nissenbaum}{1996}]{DBLP:journals/tois/FriedmanN96}
Batya Friedman and Helen Nissenbaum.
\newblock Bias in computer systems.
\newblock {\em {ACM} Trans. Inf. Syst.}, 14(3):330--347, 1996.

\bibitem[\protect\citeauthoryear{Gupta \bgroup \em et al.\egroup }{2019}]{gupta2019equalizing}
Vivek Gupta, Pegah Nokhiz, Chitradeep~Dutta Roy, and Suresh Venkatasubramanian.
\newblock Equalizing recourse across groups.
\newblock {\em arXiv preprint arXiv:1909.03166}, 2019.

\bibitem[\protect\citeauthoryear{Guyomard \bgroup \em et al.\egroup }{2023}]{guyomard2023generating}
Victor Guyomard, Fran{\c{c}}oise Fessant, Thomas Guyet, Tassadit Bouadi, and Alexandre Termier.
\newblock Generating robust counterfactual explanations.
\newblock {\em arXiv preprint arXiv:2304.12943}, 2023.

\bibitem[\protect\citeauthoryear{Hardt \bgroup \em et al.\egroup }{2016}]{hardt2016equality}
Moritz Hardt, Eric Price, and Nati Srebro.
\newblock Equality of opportunity in supervised learning.
\newblock {\em Advances in neural information processing systems}, 29, 2016.

\bibitem[\protect\citeauthoryear{Heidari \bgroup \em et al.\egroup }{2019}]{DBLP:conf/fat/HeidariLGK19}
Hoda Heidari, Michele Loi, Krishna~P. Gummadi, and Andreas Krause.
\newblock A moral framework for understanding fair {ML} through economic models of equality of opportunity.
\newblock In danah boyd and Jamie~H. Morgenstern, editors, {\em Proceedings of the Conference on Fairness, Accountability, and Transparency, FAT* 2019, Atlanta, GA, USA, January 29-31, 2019}, pages 181--190. {ACM}, 2019.

\bibitem[\protect\citeauthoryear{Hofmann}{1994}]{german_credit}
Hans Hofmann.
\newblock {Statlog (German Credit Data)}.
\newblock UCI Machine Learning Repository, 1994.
\newblock {DOI}: https://doi.org/10.24432/C5NC77.

\bibitem[\protect\citeauthoryear{Hu and Rangwala}{2020}]{hu_2020}
Qian Hu and Huzefa Rangwala.
\newblock Towards fair educational data mining: A case study on detecting at-risk students., Jun 2020.

\bibitem[\protect\citeauthoryear{Kamiran and Calders}{2009}]{kamiran2009classifying}
Faisal Kamiran and Toon Calders.
\newblock Classifying without discriminating.
\newblock In {\em 2009 2nd international conference on computer, control and communication}, pages 1--6. IEEE, 2009.

\bibitem[\protect\citeauthoryear{Kleinberg \bgroup \em et al.\egroup }{2017}]{DBLP:conf/innovations/KleinbergMR17}
Jon~M. Kleinberg, Sendhil Mullainathan, and Manish Raghavan.
\newblock Inherent trade-offs in the fair determination of risk scores.
\newblock In Christos~H. Papadimitriou, editor, {\em 8th Innovations in Theoretical Computer Science Conference, {ITCS} 2017, January 9-11, 2017, Berkeley, CA, {USA}}, volume~67 of {\em LIPIcs}, pages 43:1--43:23. Schloss Dagstuhl - Leibniz-Zentrum f{\"{u}}r Informatik, 2017.

\bibitem[\protect\citeauthoryear{K{\"o}nig \bgroup \em et al.\egroup }{2021}]{konig2021causal}
Gunnar K{\"o}nig, Timo Freiesleben, and Moritz Grosse-Wentrup.
\newblock A causal perspective on meaningful and robust algorithmic recourse.
\newblock {\em arXiv preprint arXiv:2107.07853}, 2021.

\bibitem[\protect\citeauthoryear{Mishra \bgroup \em et al.\egroup }{2021}]{mishra2021survey}
Saumitra Mishra, Sanghamitra Dutta, Jason Long, and Daniele Magazzeni.
\newblock A survey on the robustness of feature importance and counterfactual explanations.
\newblock {\em arXiv preprint arXiv:2111.00358}, 2021.

\bibitem[\protect\citeauthoryear{Obermeyer \bgroup \em et al.\egroup }{2019}]{obermeyer2019dissecting}
Ziad Obermeyer, Brian Powers, Christine Vogeli, and Sendhil Mullainathan.
\newblock Dissecting racial bias in an algorithm used to manage the health of populations.
\newblock {\em Science}, 366(6464):447--453, 2019.

\bibitem[\protect\citeauthoryear{O'Brien and Kim}{2021}]{o2021multi}
Andrew O'Brien and Edward Kim.
\newblock Multi-agent algorithmic recourse.
\newblock {\em arXiv preprint arXiv:2110.00673}, 2021.

\bibitem[\protect\citeauthoryear{O'Brien and Kim}{2022}]{o2022toward}
Andrew O'Brien and Edward Kim.
\newblock Toward multi-agent algorithmic recourse: Challenges from a game-theoretic perspective.
\newblock In {\em The International FLAIRS Conference Proceedings}, volume~35, 2022.

\bibitem[\protect\citeauthoryear{Pawelczyk \bgroup \em et al.\egroup }{2022}]{pawelczyk2022probabilistically}
Martin Pawelczyk, Teresa Datta, Johannes van-den Heuvel, Gjergji Kasneci, and Himabindu Lakkaraju.
\newblock Probabilistically robust recourse: Navigating the trade-offs between costs and robustness in algorithmic recourse.
\newblock {\em arXiv preprint arXiv:2203.06768}, 2022.

\bibitem[\protect\citeauthoryear{Rawal and Lakkaraju}{2020}]{rawal2020beyond}
Kaivalya Rawal and Himabindu Lakkaraju.
\newblock Beyond individualized recourse: Interpretable and interactive summaries of actionable recourses.
\newblock {\em Advances in Neural Information Processing Systems}, 33:12187--12198, 2020.

\bibitem[\protect\citeauthoryear{Roemer}{2002}]{Roemer2002}
John Roemer.
\newblock Equality of opportunity: A progress report.
\newblock {\em Social Choice and Welfare}, 19(2):455--471, 2002.

\bibitem[\protect\citeauthoryear{Sapiezynski \bgroup \em et al.\egroup }{2017}]{Sapiezynski2017AcademicPP}
Piotr Sapiezynski, Valentin Kassarnig, and Christo Wilson.
\newblock Academic performance prediction in a gender-imbalanced environment.
\newblock In {\em Proceedings of FATREC Workshop on Responsible Recommendation at ACM RecSys}, 2017.

\bibitem[\protect\citeauthoryear{Sharma \bgroup \em et al.\egroup }{2019}]{sharma2019certifai}
Shubham Sharma, Jette Henderson, and Joydeep Ghosh.
\newblock Certifai: Counterfactual explanations for robustness, transparency, interpretability, and fairness of artificial intelligence models.
\newblock {\em arXiv preprint arXiv:1905.07857}, 2019.

\bibitem[\protect\citeauthoryear{Upadhyay \bgroup \em et al.\egroup }{2021}]{upadhyay2021towards}
Sohini Upadhyay, Shalmali Joshi, and Himabindu Lakkaraju.
\newblock Towards robust and reliable algorithmic recourse.
\newblock {\em Advances in Neural Information Processing Systems}, 34:16926--16937, 2021.

\bibitem[\protect\citeauthoryear{Ustun \bgroup \em et al.\egroup }{2019}]{ustun2019actionable}
Berk Ustun, Alexander Spangher, and Yang Liu.
\newblock Actionable recourse in linear classification.
\newblock In {\em Proceedings of the conference on fairness, accountability, and transparency}, pages 10--19, 2019.

\bibitem[\protect\citeauthoryear{von K{\"u}gelgen \bgroup \em et al.\egroup }{2022}]{von2022fairness}
Julius von K{\"u}gelgen, Amir-Hossein Karimi, Umang Bhatt, Isabel Valera, Adrian Weller, and Bernhard Sch{\"o}lkopf.
\newblock On the fairness of causal algorithmic recourse.
\newblock In {\em Proceedings of the AAAI Conference on Artificial Intelligence}, volume~36, pages 9584--9594, 2022.

\bibitem[\protect\citeauthoryear{Yang and Stoyanovich}{2017}]{yang2017measuring}
Ke~Yang and Julia Stoyanovich.
\newblock Measuring fairness in ranked outputs.
\newblock In {\em Proceedings of the 29th international conference on scientific and statistical database management}, pages 1--6, 2017.

\bibitem[\protect\citeauthoryear{Zehlike \bgroup \em et al.\egroup }{2022a}]{zehlike2022fairness}
Meike Zehlike, Ke~Yang, and Julia Stoyanovich.
\newblock Fairness in ranking, part i: Score-based ranking.
\newblock {\em ACM Computing Surveys (CSUR)}, 2022.

\bibitem[\protect\citeauthoryear{Zehlike \bgroup \em et al.\egroup }{2022b}]{zehlike2022fairness2}
Meike Zehlike, Ke~Yang, and Julia Stoyanovich.
\newblock Fairness in ranking, part ii: Learning-to-rank and recommender systems.
\newblock {\em ACM Computing Surveys (CSUR)}, 2022.

\end{thebibliography}
\newpage 
\appendix
\section{Supplementary Materials}

\subsection{Additional background}

Below is the algorithm~\citet{DBLP:conf/eaamo/FonsecaBABS23} use for simulating mutli-agent recourse over time. It is valuable background for understanding the mechanics of the simulation.

\begin{algorithm}[]
    \caption{Simulating multi-agent recourse over time}
    \label{alg:recourse_over_time}
    \begin{algorithmic}[1]
    \Require Populations of agents $P^a$ and $P^d$; classifier $f:\mathcal{X} \rightarrow [0,1]$; function for generating recourse recommendations $r:\mathcal{X} \rightarrow \mathcal{X}$; function describing agent behavior $a_{e}:\mathcal{X} \times \mathcal{X} \rightarrow \mathcal{X}$;  number of time steps $T$, number of favorable outcomes per time-step $k$, number of new agents at each time-step $|N_t|$
    \Ensure Sequence of sets of agents $\{P_{0}, \dots, P_{T}\}$ in their final state for each time step.
    \State $t \gets 0$
    \State $P_0 \gets P^a_0 \cup P^d_0$
    \While{$t<T$}
        \If{$t \neq 0$}
            \State $N_t \gets$ create a new set of agents that are joining the simulation (half of these agents are in $P^a$ and half are in $P^d$);
            \State $P_t \gets A_t \cup N_t$, union of sets to create full population; 
        \EndIf
        
        \State $S_t \gets \{f(x_t) | \forall x_t \in P_t, f(x_t)\}$, compute a score for each agent; 
        \State $\bar{P}_t \gets sort(P_t,S_t)$, sort agents by score; 
        \State $P^k_t \gets select\_top\_k(\bar{P}_t)$, select the top-$k$ highest scoring agents, and assign them a positive outcome (assign the rest of the agents a negative outcome);
        \State $s_t \gets update(P^k_t)$, update the decision threshold to be the minimum score from $P^k_t$; 
        \State $L_t \gets P_t \setminus P^k_t$, select the agents who recieved a negative outcome;
        \State $R_t \gets \{x'| \forall x \in L_t, r(x)=x'\}$, generate recourse recommendations; 
        \State $A_{t+1} \gets \{x_{new}| \forall x \in L_t, a_{e_x}(x,x') = x_{new}\}$, for each agent, determine what action they take (this includes taking no action) and prepare to move to them the next time-step
        \State $t \gets t + 1$, increase time-step
    \EndWhile
    \end{algorithmic}
\end{algorithm}

\subsection{Additional details about metrics}

\begin{proposition}\label{prop:ineq}
Given populations $P^a$ and $P^d$ and:
\begin{enumerate}
    \item Initial qualifications: $\forall x \in P^a$, each feature $j$ is given by $x^{(j)} \sim N(\mu^{(j)}_{a} , \sigma^2)$, and $\forall x \in P^d$, each feature is given by $x^{(j)} \sim N(\mu^{(j)}_{d} , \sigma^2)$.
    \item Effort is an i.i.d. random variable with mean $e_a$ for $P^a$, and $e_d$ for $P^d$, respectively.
\end{enumerate}
The metric $rETR_t$, per Eq.~\ref{eq:etr_ratio}, is proportional to:

\begin{equation}
    \frac{ETR_{t}(P^d)}{ETR_{t}(P^a)} \propto \frac{e_d}{e_a}
\end{equation}
\end{proposition}

\begin{proof}
Given the definition of $ETR$ in \ref{eq:etr} and the cost formula \ref{eq:cost}, the $ETR$ for $P^a_t$ is:
$$
ETR_{P^a_t} = \frac{1}{|\bar P^a_t|} \sum_{x \in \bar P^a_t} C_t(x) = \frac{1}{|\bar P^a_t|} \sum_{x \in \bar P^a_t} \sum_{\forall i < t} c(x_{i},x_{i+1}).
$$
Given that $ETR$ is calculated \textit{ex-post}, we can assume that the cost $c(x_{i},x_{i+1})$ in a timestep for any agent is equal to the exerted effort, that is:
$$
c(x_{i},x_{i+1}) = a(x_{i},x_{i+1}).
$$
Therefore, summing over time and for all agents the metric can be written as:
$$
ETR_{P^a_t} = \sum_{\forall i < t} \frac{1}{|\bar P^a_t|} \sum_{x \in \bar P^a_t} a(x_{i},x_{i+1}).
$$

Given assumption 3, the exerted effort for all the agents in a population $\forall x \in P^a$ is identical distributed and independent random variable with mean $e_a$. Therefore, for the law of the large numbers the sample average ${|\bar P^a_t|} \sum_{x \in \bar P^a_t} a(x_{i},x_{i+1})$ converges to the expected value $e_a$.

\begin{align*}
\frac{1}{|\bar P^a_t|} \sum_{x \in \bar P^a_t} a(x_{i},x_{i+1}) \propto
\frac{1}{|\bar P^a_t|} \sum_{x \in \bar P^a_t} e_a  = \frac{|\bar P^a_t|}{|\bar P^a_t|} e_a  = e_a  
\end{align*}
Finally, the $ETR$ for population $P^a$ can be simplified as:
$$
    ETR_{P^a_t} \propto \sum_{\forall i < t} e_a = t * e_a
$$
    
The ratio of the metric between the two population is proportional to:
$$
\frac{ETR_{P^a_t}}{ETR_{P^d_t}} \propto \frac{t * e_d}{t * e_d} = \frac{e_a}{e_d}
$$

\end{proof}

Finally, with similar consideration can be proof that the time-to-recourse $TTR$ metric is proportional to population-level descriptors such as the exerted effort $e$, the initial qualification parameters (i.e. $\mu$), and the number of individuals at each time.

\subsection{Additional details about experiments and their results}

In this section we have included additional experimental results that could not fit in the main body of the paper. Each table and figure has a detailed caption describing what is shown, but we summarizes each artifact here. Table~\ref{tab:suppl_results} shows extended results that could not be shown in Table~\ref{tab:results} from the main body of the paper. Figure~\ref{fig:boundary_shift} is the expanded version of Figure~\ref{fig:boundary_shift_one} from the main body of the paper. Figures~\ref{fig:suppl1},~\ref{fig:suppl2},~\ref{fig:suppl3} and~\ref{fig:suppl4} show the final \emph{distribution} of agents for various experiments.

\begin{table*}[]
\small
\centering
\caption{The effort-to-recourse disparity ($rETR)$ and disparate time-to-recourse ($dTTR$) between an advantaged population $P_a$ (used as the reference population) and a disadvantaged population $P_d$ under different initial circumstances as described by the qualification $q$ and 3 effort conditions: when both populations have the same distribution, when the expected effort of $P_a$ is twice that of $P_b$, and the reverse of the latter. Table (a) shows results without mitigation, (b) with group recourse regulraization~\cite{gupta2019equalizing}, and (c) with the combined fair selection and CDA strategy~\cite{ferrario2022robustness}. All disparities are calculated cumulatively from the initial time-step to $T$.\\ *\emph{Due to run-time constraints, the group recourse regularization method was only run over 10 different ranodm seeds.}}
\label{tab:suppl_results}

\begin{tabular}{cc|ccc|ccc}
\textbf{} & \textbf{} & \multicolumn{3}{c|}{\textbf{$rETR$}} & \multicolumn{3}{c}{\textbf{$dTTR$}} \\
\textbf{} & \textbf{q} & \textbf{Baseline} & \textbf{\begin{tabular}[c]{@{}c@{}}Group Recourse\\ Regularization*\end{tabular}} & \textbf{\begin{tabular}[c]{@{}c@{}}Combined\\ (Fair Selection\\ and CDA)\end{tabular}} & \textbf{Baseline} & \textbf{\begin{tabular}[c]{@{}c@{}}Group Recourse\\ Regularization*\end{tabular}} & \textbf{\begin{tabular}[c]{@{}c@{}}Combined\\ (Fair Selection\\ and CDA)\end{tabular}} \\ \hline
\multirow{4}{*}{$e_a = e_d$} & 0 & 1.001 $\pm$ 0.003 & 0.941 $\pm$ 0.024 & 0.999 $\pm$ 0.001 & -0.043 $\pm$ 0.022 & -0.271 $\pm$ 0.093 & 0.004 $\pm$ 0.012 \\
 & 1 & 1.213 $\pm$ 0.003 & 0.736 $\pm$ 0.150 & 1.044 $\pm$ 0.002 & 0.646 $\pm$ 0.023 & -0.391 $\pm$ 0.118 & 0.226 $\pm$ 0.015 \\
 & 2 & 1.429 $\pm$ 0.004 & 0.035 $\pm$ 0.224 & 1.020 $\pm$ 0.005 & 1.177 $\pm$ 0.024 & -0.309 $\pm$ 0.252 & 0.220 $\pm$ 0.017 \\
 & 3 & 1.640 $\pm$ 0.006 & 0.936 $\pm$ 1.429 & 1.063 $\pm$ 0.004 & 1.573 $\pm$ 0.024 & -0.397 $\pm$ 0.266 & 0.021 $\pm$ 0.014 \\ \hline
\multirow{4}{*}{$e_a > e_d$} & 0 & 0.942 $\pm$ 0.002 & 0.925 $\pm$ 0.025 & 0.940 $\pm$ 0.001 & 1.315 $\pm$ 0.021 & 0.726 $\pm$ 0.113 & 0.181 $\pm$ 0.014 \\
 & 1 & 1.109 $\pm$ 0.003 & 0.710 $\pm$ 0.063 & 0.916 $\pm$ 0.004 & 1.835 $\pm$ 0.023 & 0.255 $\pm$ 0.142 & 0.182 $\pm$ 0.012 \\
 & 2 & 1.275 $\pm$ 0.003 & 0.385 $\pm$ 0.090 & 0.901 $\pm$ 0.005 & 2.257 $\pm$ 0.023 & -0.118 $\pm$ 0.169 & 0.046 $\pm$ 0.015 \\
 & 3 & 1.429 $\pm$ 0.004 & -0.021 $\pm$ 0.195 & 0.951 $\pm$ 0.003 & 2.532 $\pm$ 0.023 & -0.027 $\pm$ 0.160 & -0.050 $\pm$ 0.013 \\ \hline
\multirow{4}{*}{$e_a < e_d$} & 0 & 1.064 $\pm$ 0.002 & 1.036 $\pm$ 0.027 & 1.064 $\pm$ 0.002 & -1.338 $\pm$ 0.021 & -1.329 $\pm$ 0.116 & -0.199 $\pm$ 0.013 \\
 & 1 & 1.252 $\pm$ 0.003 & 0.938 $\pm$ 0.332 & 1.090 $\pm$ 0.001 & -0.845 $\pm$ 0.021 & -1.264 $\pm$ 0.149 & -0.127 $\pm$ 0.014 \\
 & 2 & 1.433 $\pm$ 0.004 & 0.504 $\pm$ 0.332 & 1.127 $\pm$ 0.002 & -0.446 $\pm$ 0.024 & -0.652 $\pm$ 0.170 & -0.073 $\pm$ 0.016 \\
 & 3 & 1.605 $\pm$ 0.005 & -0.623 $\pm$ 0.696 & 1.175 $\pm$ 0.003 & -0.168 $\pm$ 0.024 & -0.490 $\pm$ 0.118 & 0.030 $\pm$ 0.019
\end{tabular}

\end{table*}

\begin{figure*}
     \centering
     \begin{subfigure}[t]{0.32\textwidth}
         \centering
         \includegraphics[width=\textwidth]{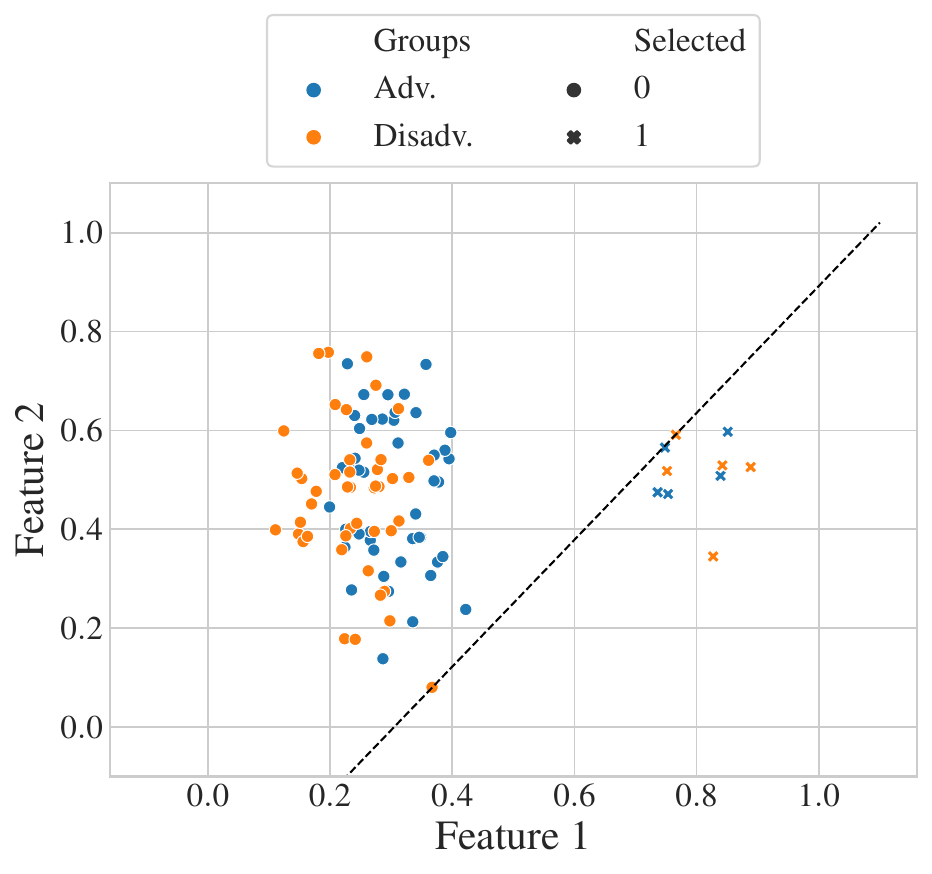}
         \caption{$t=0$}
     \end{subfigure}
     \begin{subfigure}[t]{0.32\textwidth}
         \centering
         \includegraphics[width=\textwidth]{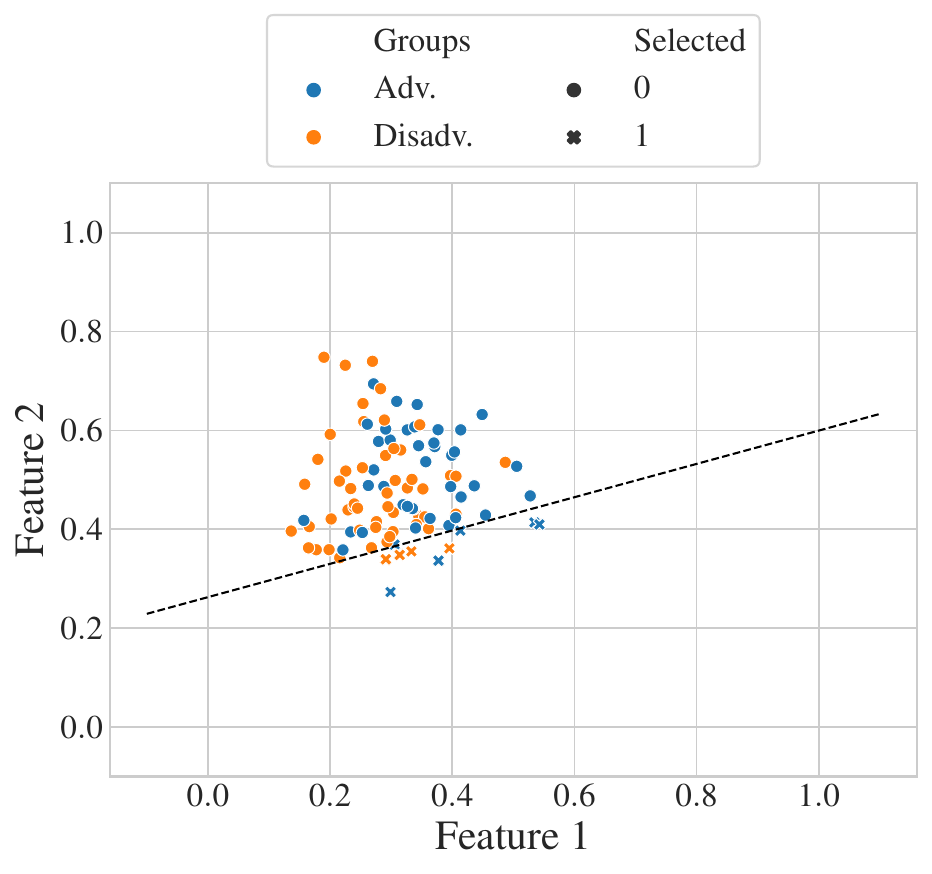}
         \caption{$t=5$}
     \end{subfigure}
     \begin{subfigure}[t]{0.32\textwidth}
         \centering
         \includegraphics[width=\textwidth]{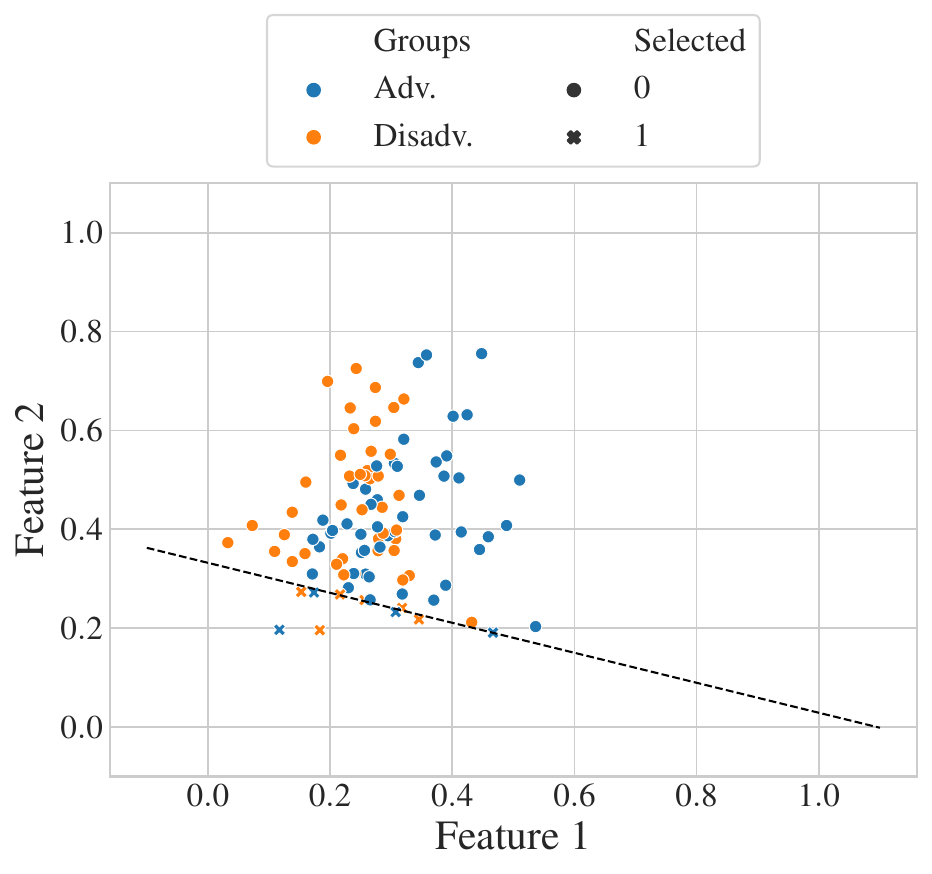}
         \caption{$t=15$}
     \end{subfigure}

     \caption{The figures shows agents' feature values and the decision boundary for a positive outcome (the dashed line) for 3 time-steps from a single run of the simulation using the in-processing Group Recourse Regularization method~\cite{gupta2019equalizing} to mitigate unfairness in recourse. An unintended consequence of the mitigation strategy is the dramatic movement of the decision boundary. At $t=0$, recourse implied that agents would need to increase Feature value 1 to cross the decision boundary for a positive outcome, but at $t=15$ recourse implies that both feature values should be decreased.}
     \label{fig:boundary_shift}
    
\end{figure*}

\newcommand{\rulesep}{\unskip\ \vrule\ }

\begin{figure*}
     \centering
     \begin{subfigure}[t]{0.24\textwidth}
         \centering
         \includegraphics[width=\textwidth]{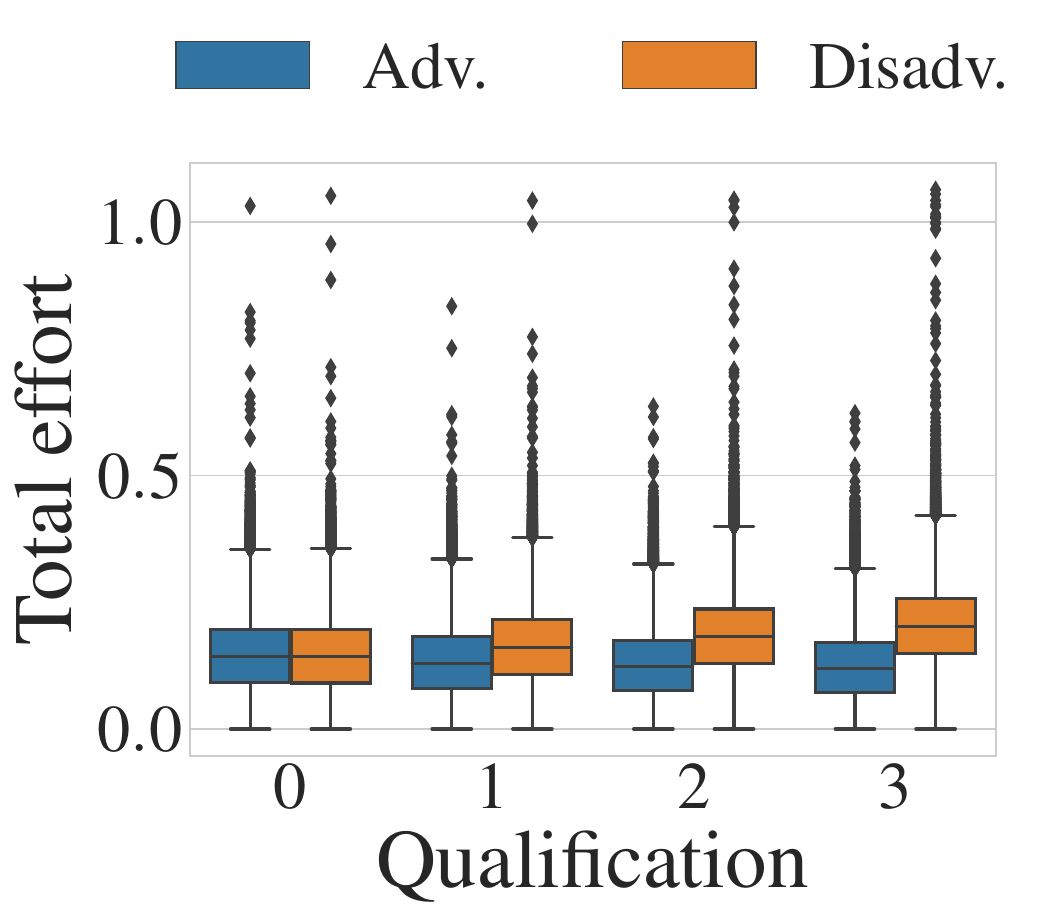}
         \caption{Total effort \\ Without mitigation}
     \end{subfigure}
     \begin{subfigure}[t]{0.24\textwidth}
         \centering
         \includegraphics[width=\textwidth]{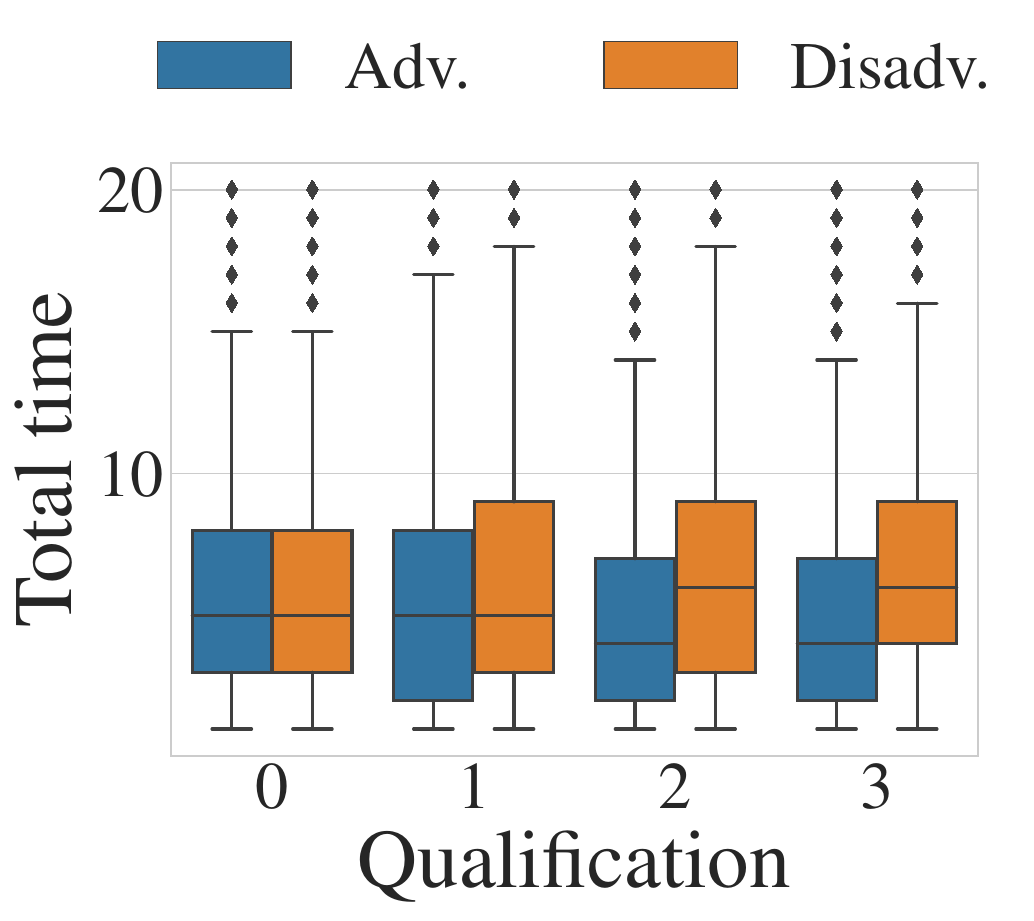}
         \caption{Total time\\Without mitigation}
     \end{subfigure}
     \rulesep
     \begin{subfigure}[t]{0.24\textwidth}
         \centering
         \includegraphics[width=\textwidth]{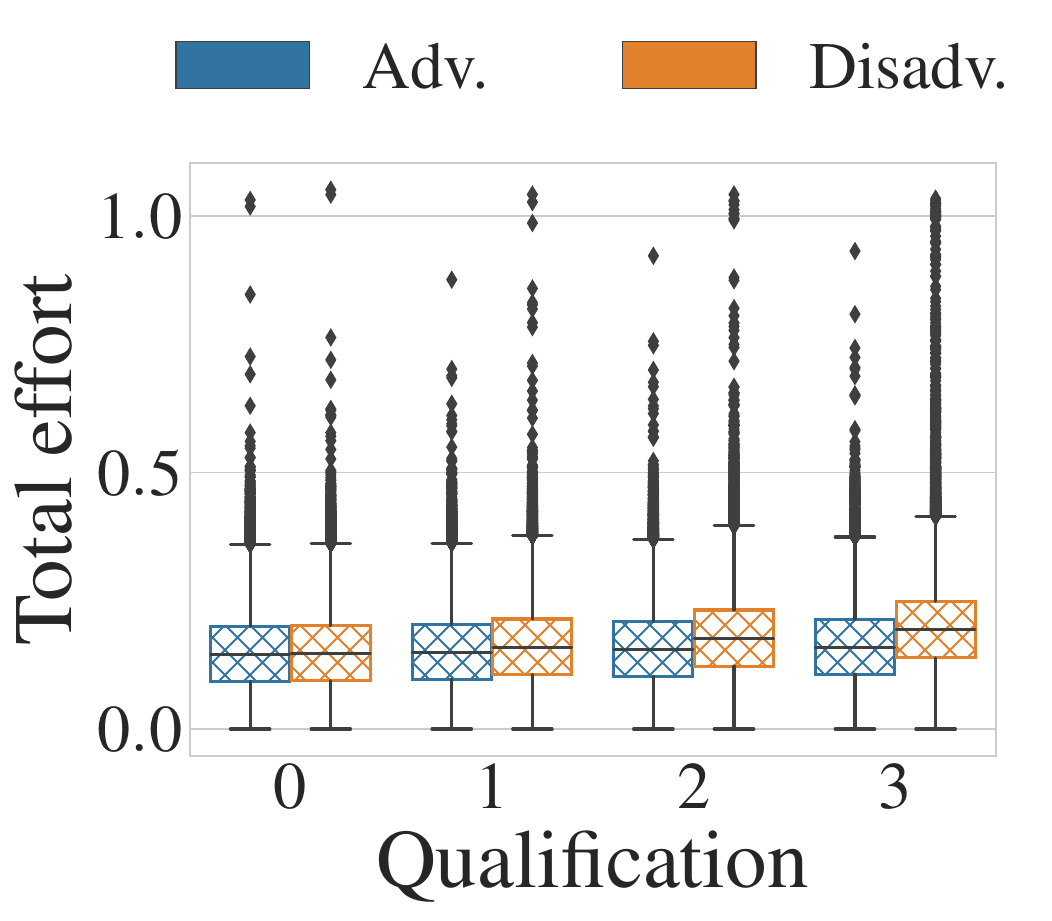}
         \caption{Total effort\\With mitigation}
     \end{subfigure}
     \begin{subfigure}[t]{0.24\textwidth}
         \centering
         \includegraphics[width=\textwidth]{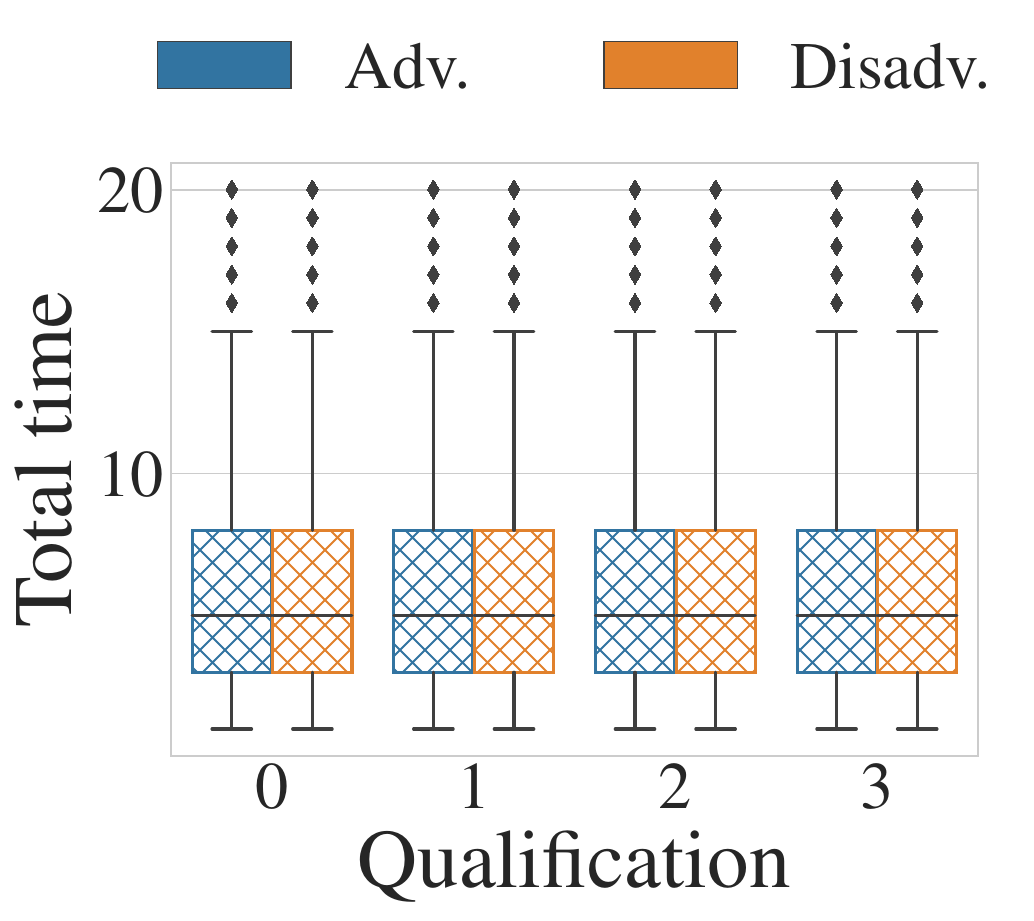}
         \caption{Total time\\With mitigation}
     \end{subfigure}

     \caption{Individual agents' total effort (figures \emph{a} and {c}) and total time to receive a positive outcome (figures \emph{b} and \emph{d}) by qualification $q$. Shown both without mitigation (figures \emph{a} and \emph{b}) and with mitigation (figures \emph{c} and \emph{d}). In all figures, $e_a = e_d$, meaning that in expectation, the disadvantaged population exerts twice as much effort as the advantaged population.}
     \label{fig:suppl1}
\end{figure*}

\begin{figure*}
     \centering
     \begin{subfigure}[t]{0.24\textwidth}
         \centering
         \includegraphics[width=\textwidth]{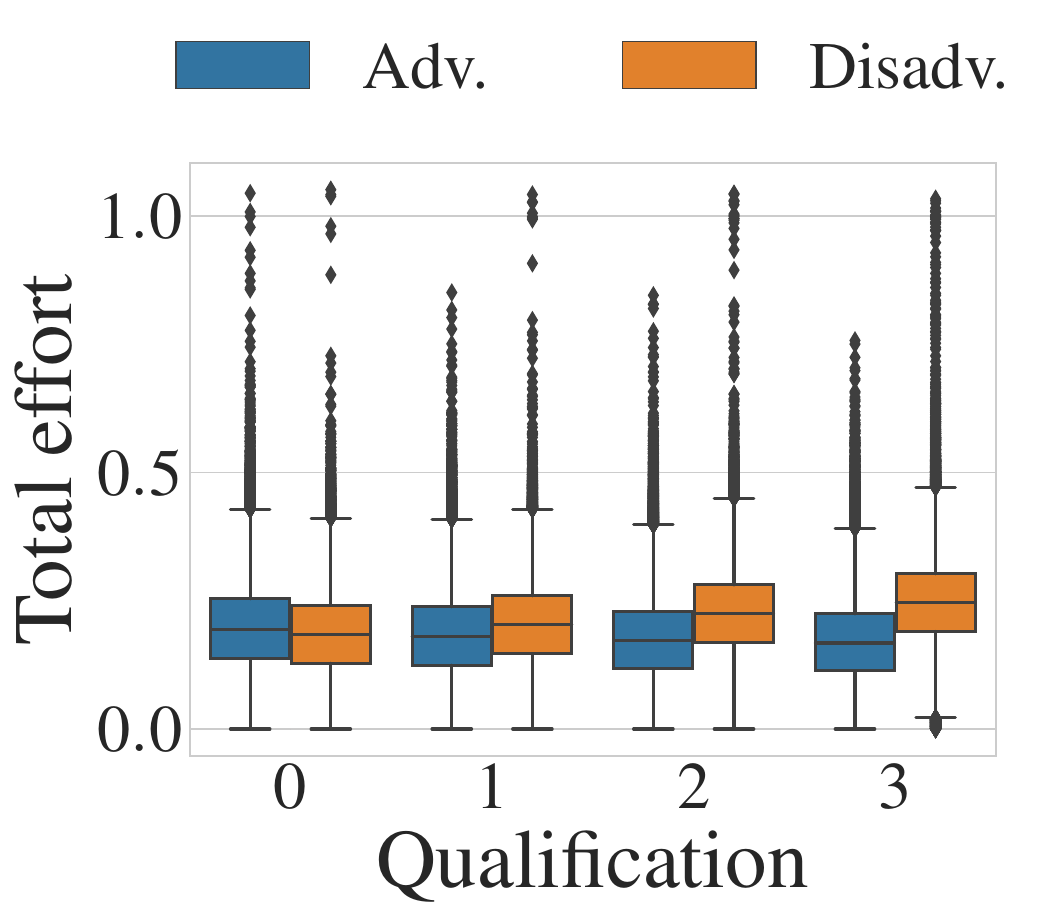}
         \caption{Total effort \\ Without mitigation}
     \end{subfigure}
     \begin{subfigure}[t]{0.24\textwidth}
         \centering
         \includegraphics[width=\textwidth]{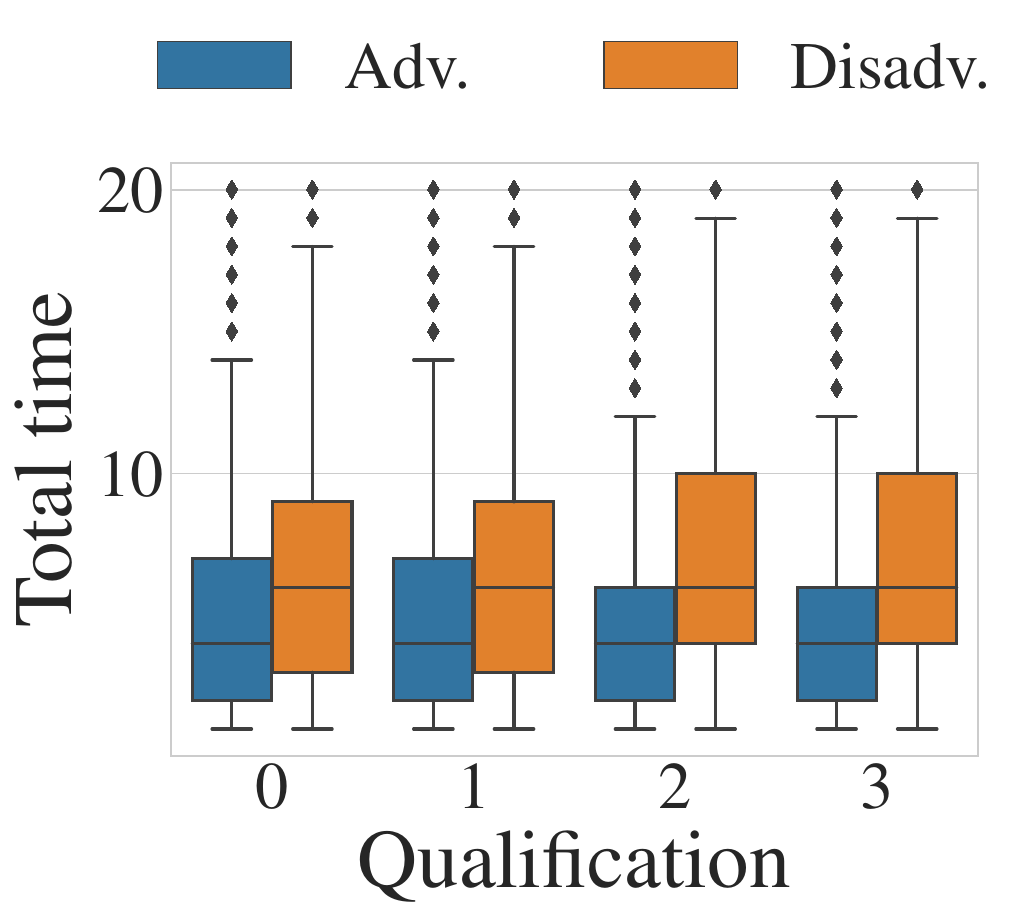}
         \caption{Total time\\Without mitigation}
     \end{subfigure}
     \rulesep
     \begin{subfigure}[t]{0.24\textwidth}
         \centering
         \includegraphics[width=\textwidth]{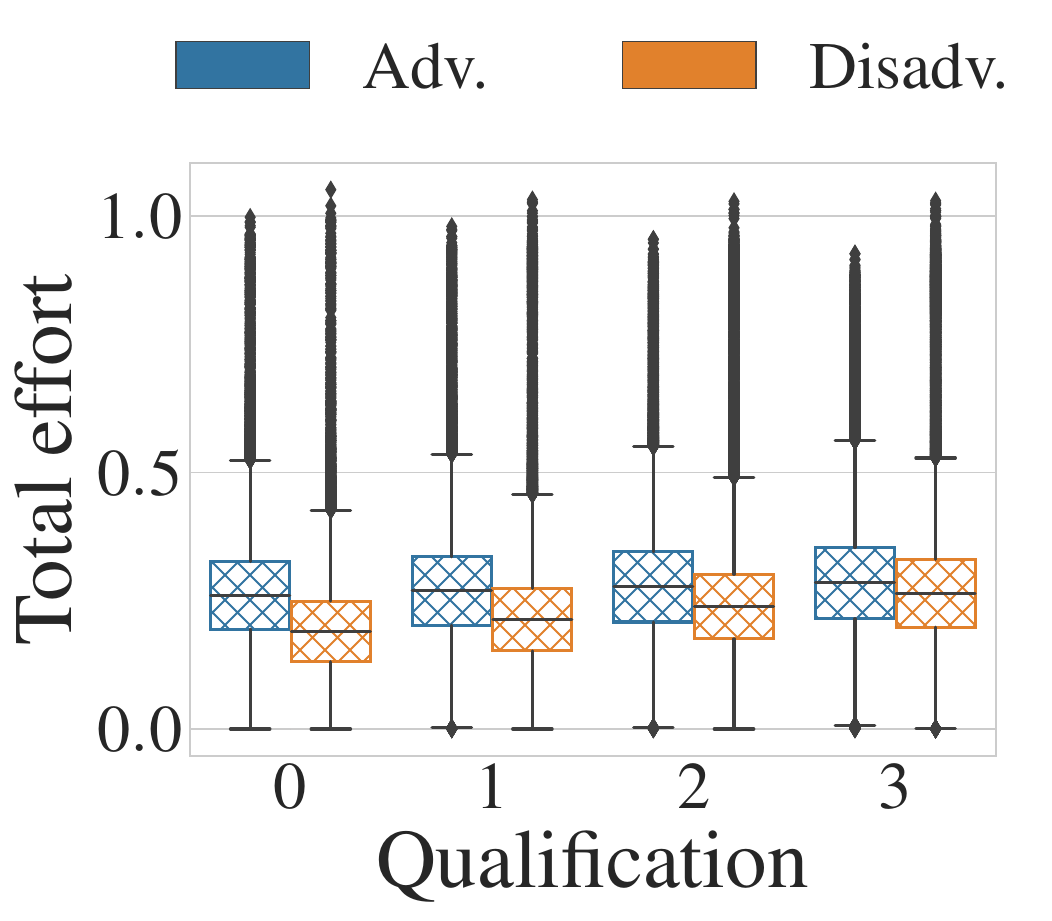}
         \caption{Total effort\\With mitigation}
     \end{subfigure}
     \begin{subfigure}[t]{0.24\textwidth}
         \centering
         \includegraphics[width=\textwidth]{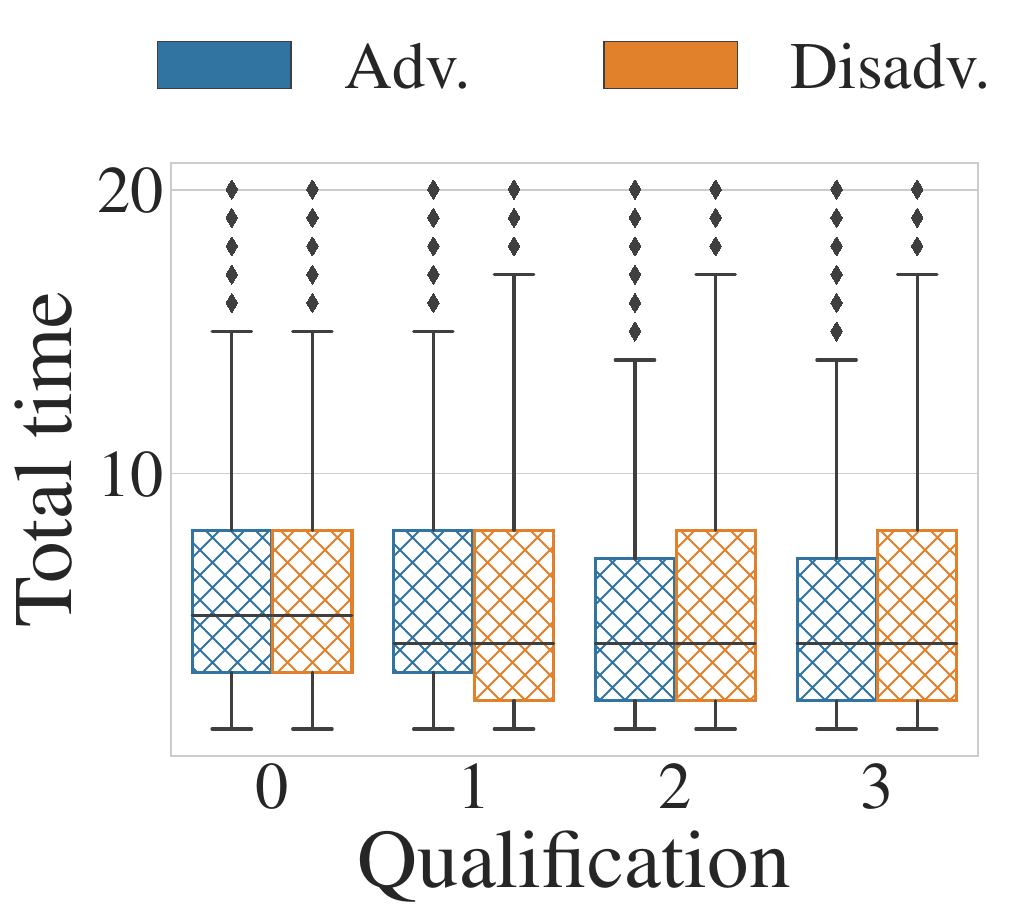}
         \caption{Total time\\With mitigation}
     \end{subfigure}

     \caption{Individual agents' total effort (figures \emph{a} and {c}) and total time to receive a positive outcome (figures \emph{b} and \emph{d}) by qualification $q$. Shown both without mitigation (figures \emph{a} and \emph{b}) and with mitigation (figures \emph{c} and \emph{d}). In all figures, $e_a > e_d$, meaning that in expectation, the disadvantaged population exerts twice as much effort as the advantaged population.}
     \label{fig:suppl2}
\end{figure*}

\begin{figure*}
\captionsetup[subfigure]{justification=centering}
     \centering
     \begin{subfigure}[t]{0.24\textwidth}
         \centering
         \includegraphics[width=\textwidth]{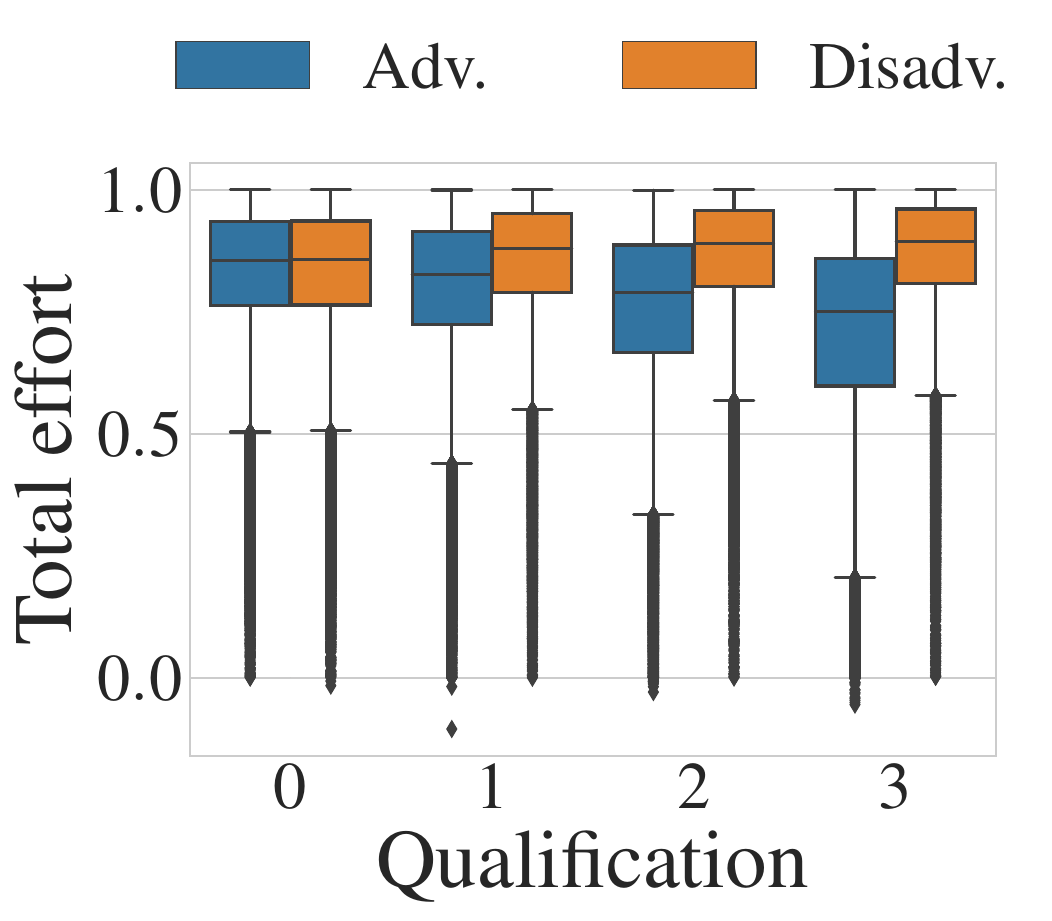}
         \caption{Total effort, $e_a = e_d$,  \\ with benchmarking mitigation strategy}
     \end{subfigure}
     \hspace{1.5cm}
     \begin{subfigure}[t]{0.24\textwidth}
         \centering
         \includegraphics[width=\textwidth]{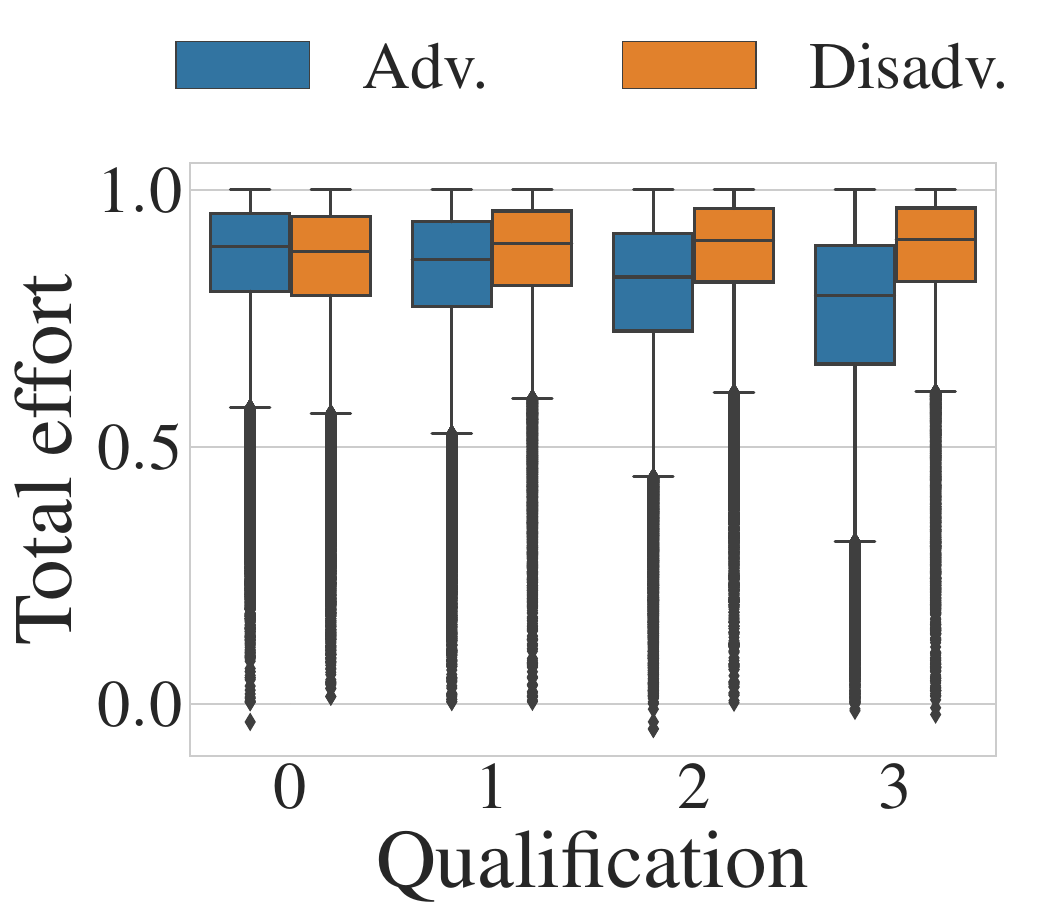}
         \caption{Total effort, $e_a > e_d$,  \\ with benchmarking mitigation strategy}
     \end{subfigure}
     \hspace{1.5cm}
     \begin{subfigure}[t]{0.24\textwidth}
         \centering
         \includegraphics[width=\textwidth]{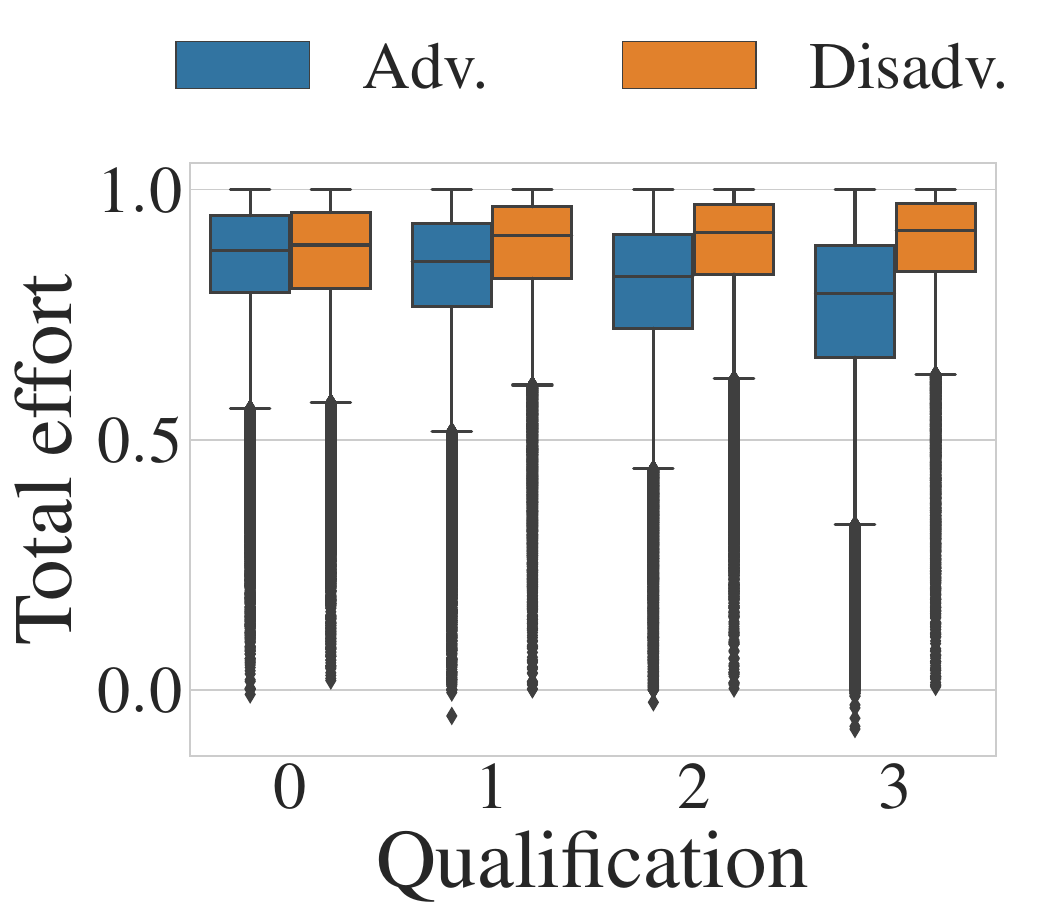}
         \caption{Total effort, $e_a < e_d$, \\With mitigation}
     \end{subfigure}

     \vspace{5mm}
     \begin{subfigure}[t]{0.24\textwidth}
         \centering
         \includegraphics[width=\textwidth]{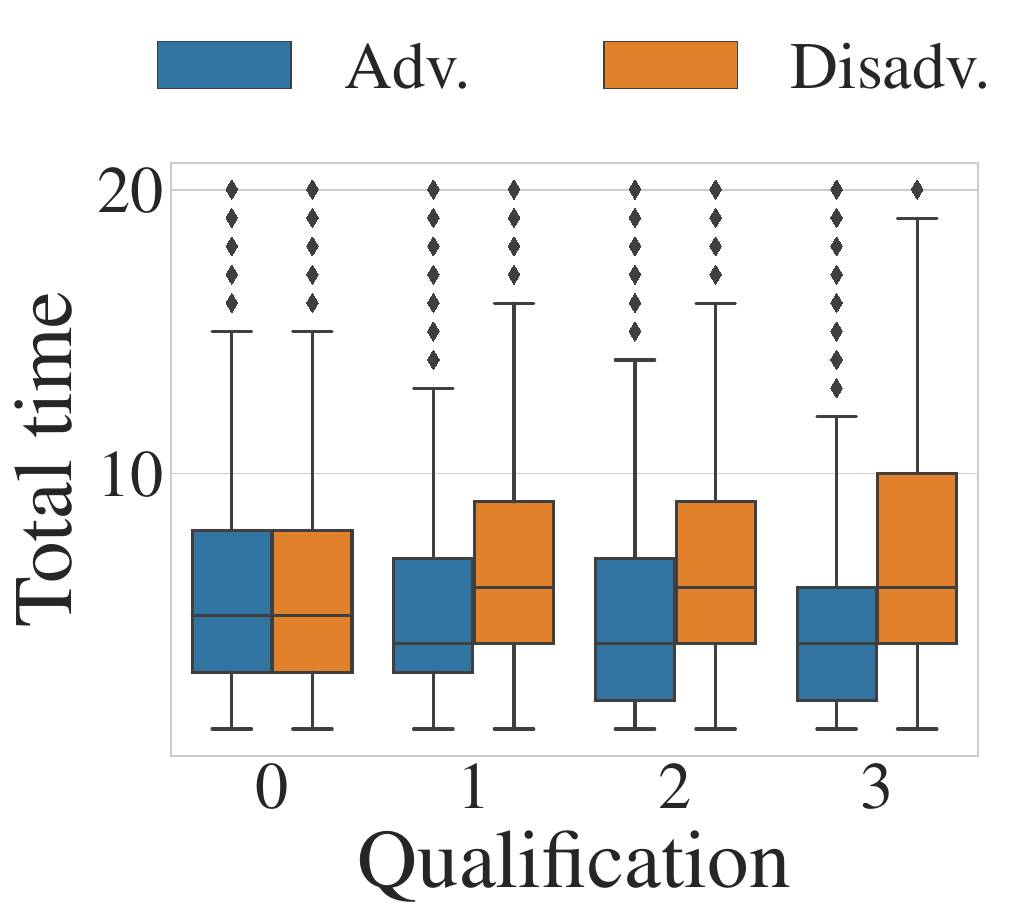}
         \caption{Total time, $e_a = e_d$, \\ with benchmarking mitigation strategy}
     \end{subfigure}
     \hspace{1.5cm}
     \begin{subfigure}[t]{0.24\textwidth}
         \centering
         \includegraphics[width=\textwidth]{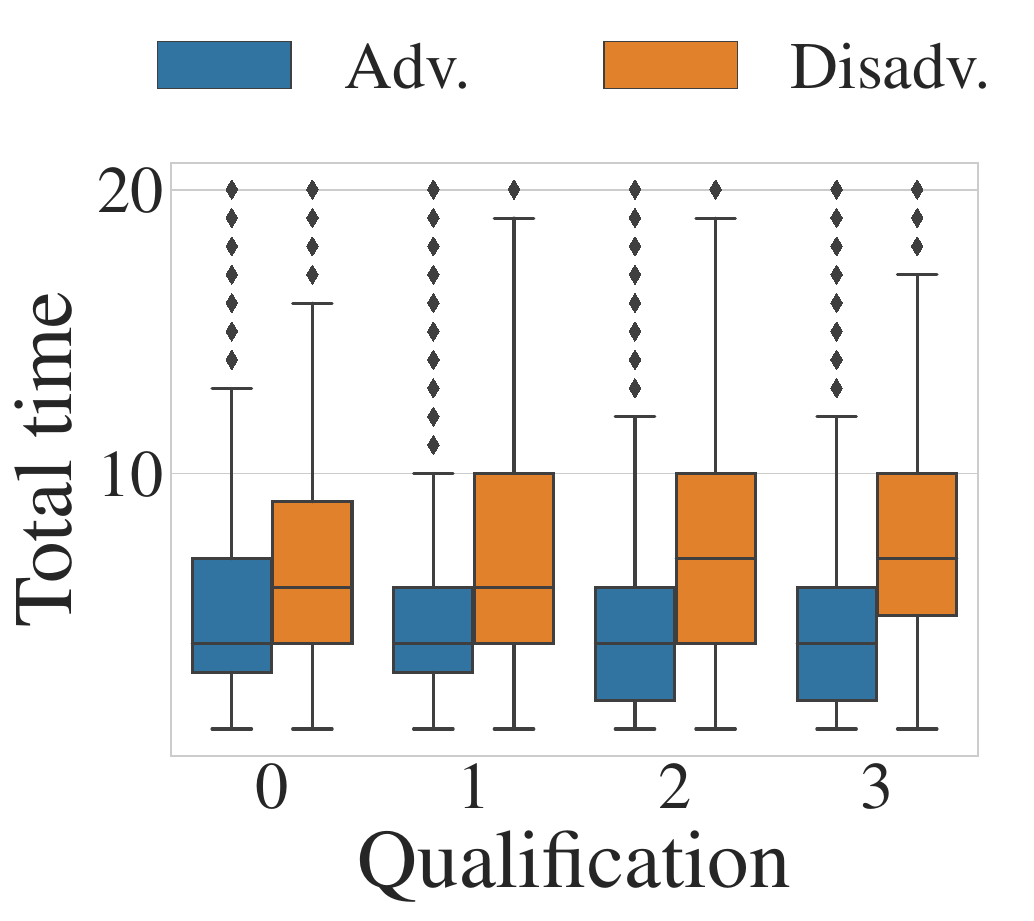}
         \caption{Total time, $e_a > e_d$, \\ with benchmarking mitigation strategy}
     \end{subfigure}
     \hspace{1.5cm}
     \begin{subfigure}[t]{0.24\textwidth}
         \centering
         \includegraphics[width=\textwidth]{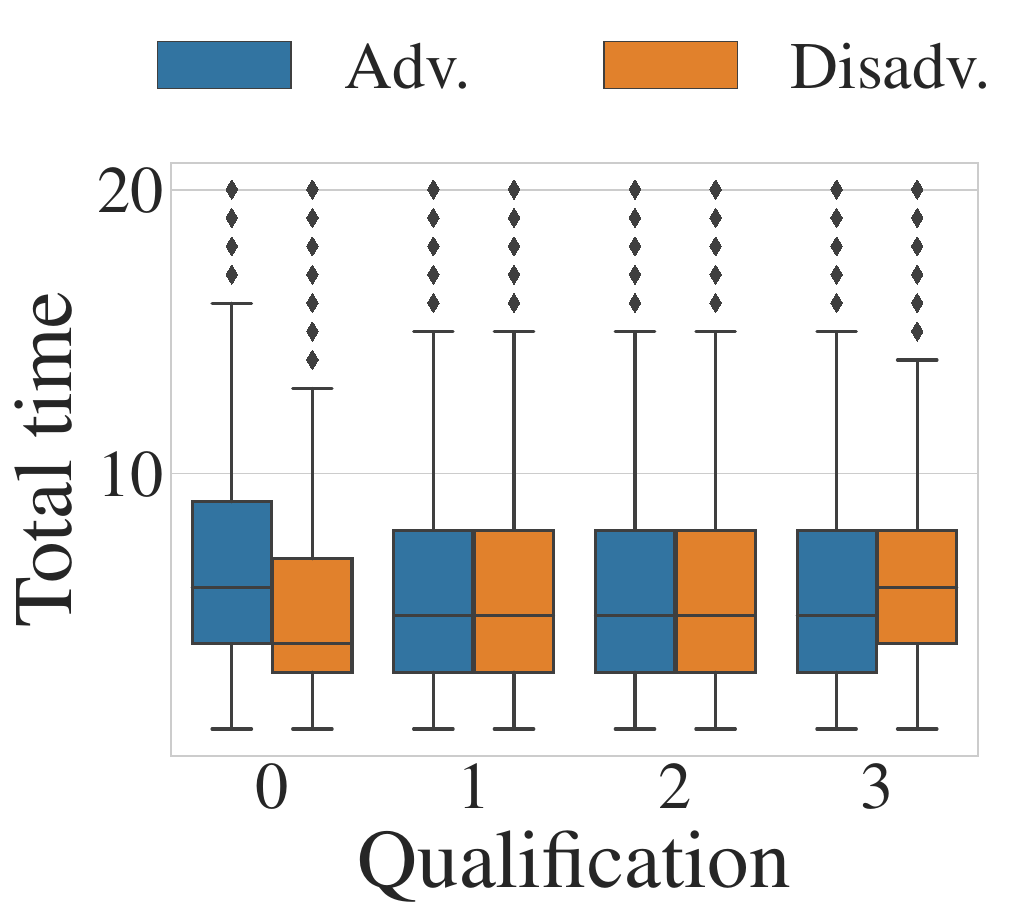}
         \caption{Total time, $e_a < e_d$, \\With mitigation}
     \end{subfigure}

     \caption{Individual agents' total effort (figures \emph{a-c}) and total time to receive a positive outcome (figures \emph{d-f}) by qualification $q$ when using the benchmarking mitigation strategy.}
     \label{fig:suppl3}
\end{figure*}

\begin{figure*}
     \centering
     \begin{subfigure}[t]{0.24\textwidth}
         \centering
         \includegraphics[width=\textwidth]{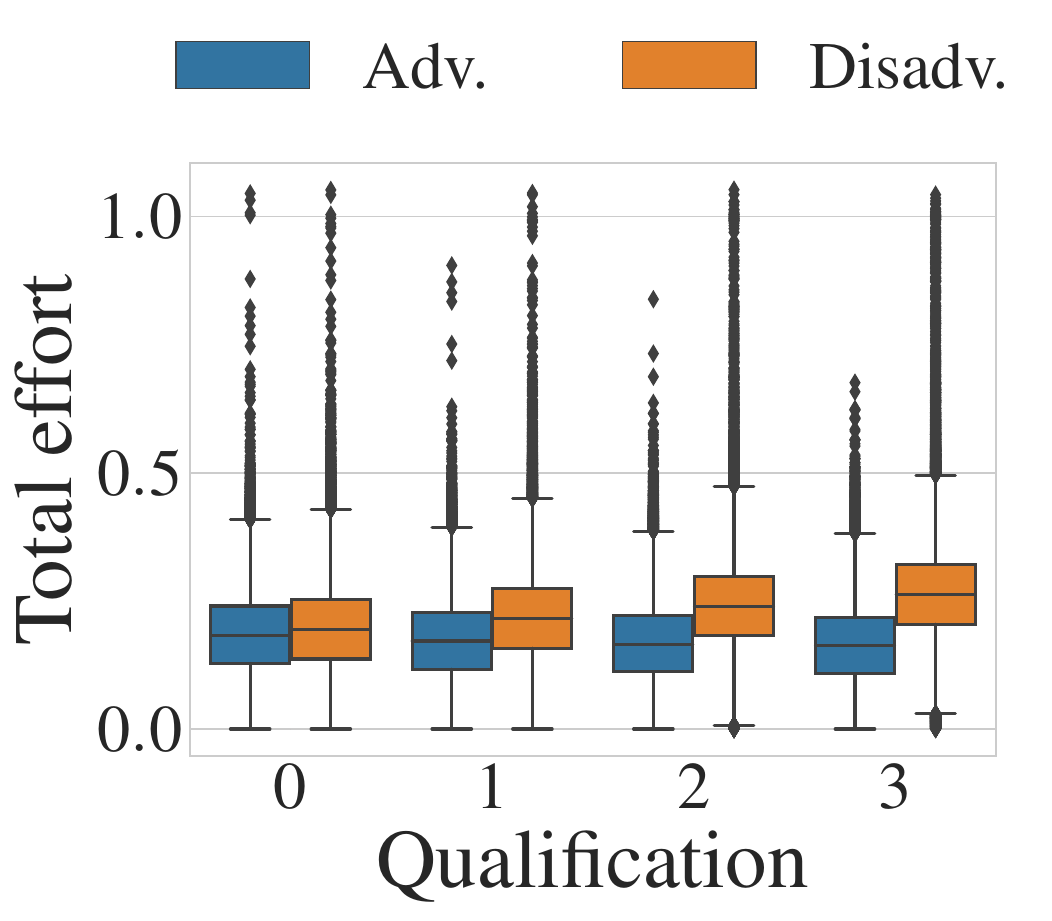}
         \caption{Total effort \\ Without mitigation}
     \end{subfigure}
     \begin{subfigure}[t]{0.24\textwidth}
         \centering
         \includegraphics[width=\textwidth]{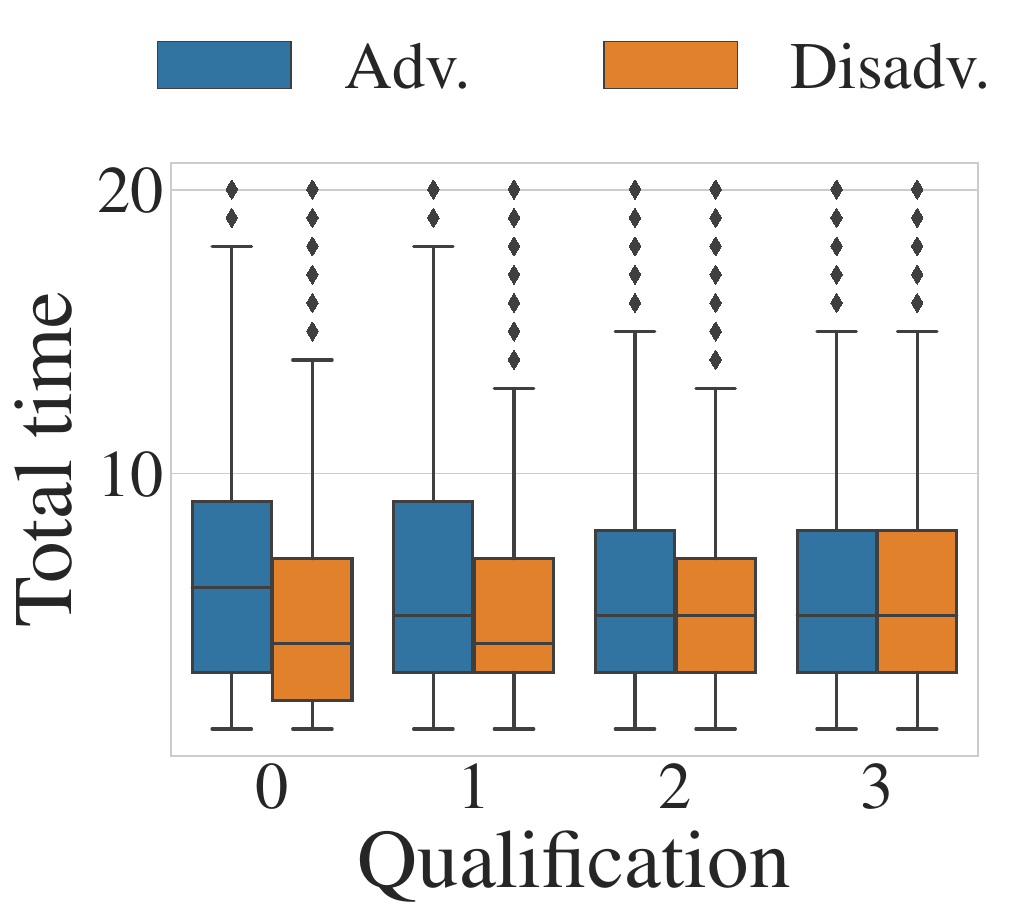}
         \caption{Total time\\Without mitigation}
     \end{subfigure}
     \rulesep
     \begin{subfigure}[t]{0.24\textwidth}
         \centering
         \includegraphics[width=\textwidth]{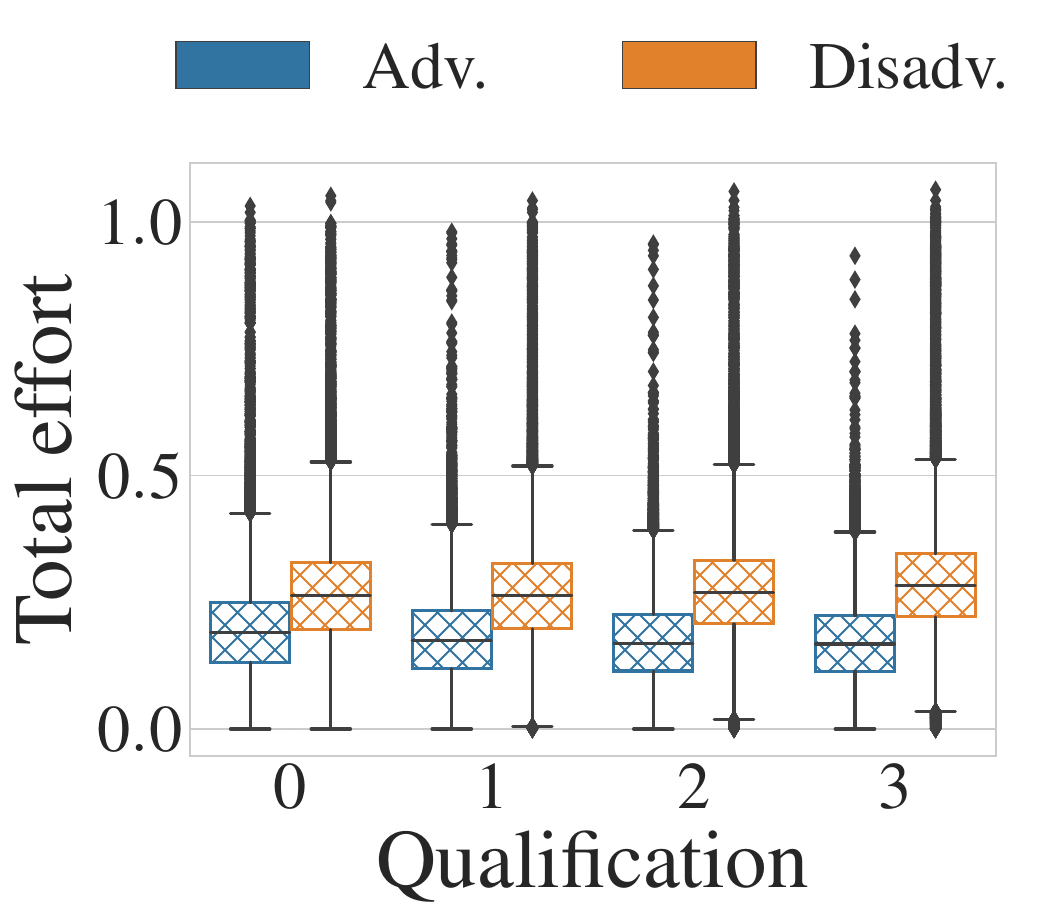}
         \caption{Total effort\\With mitigation}
     \end{subfigure}
     \begin{subfigure}[t]{0.24\textwidth}
         \centering
         \includegraphics[width=\textwidth]{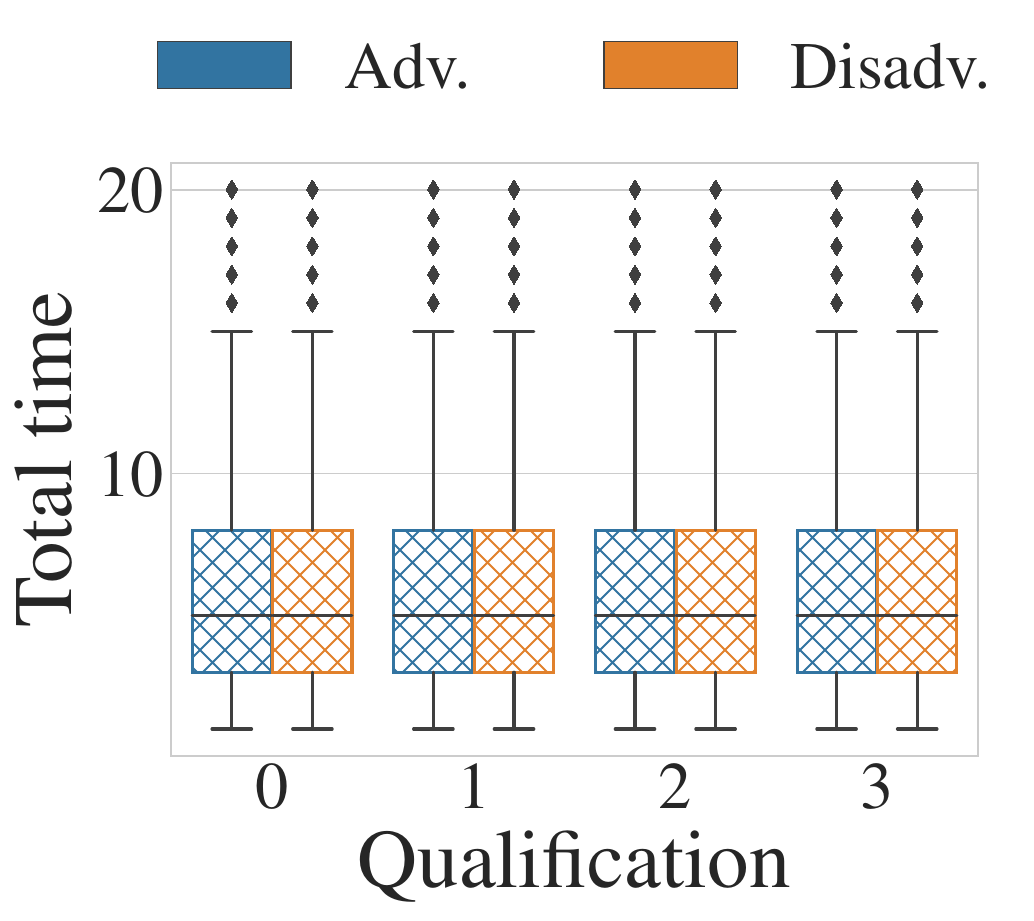}
         \caption{Total time\\With mitigation}
     \end{subfigure}

     \caption{Individual agents' total effort (figures \emph{a} and {c}) and total time to receive a positive outcome (figures \emph{b} and \emph{d}) by qualification $q$. Shown both without mitigation (figures \emph{a} and \emph{b}) and with \cns mitigation (figures \emph{c} and \emph{d}). In all figures, $e_a < e_b$, meaning that in expectation, the disadvantaged population exerts twice as much effort as the advantaged population. The figures show that despite exerting more total effort (\emph{a}), agents in the disadvantaged population still take nearly the same time (\emph{b}) to attain a positive outcome, illustrating that they cannot overcome disparities in initial circumstances without intervention (\emph{c}), (\emph{d}).}
     \label{fig:suppl4}
\end{figure*}

\end{document}